\documentclass[letterpaper]{article}
\usepackage{aaai19}
\usepackage{times}
\usepackage{helvet}
\usepackage{courier}
\usepackage{url}
\usepackage{graphicx}
\usepackage{graphicx}
\usepackage{balance}  
\usepackage{url}            
\usepackage{booktabs}       
\usepackage{amsfonts}       
\usepackage{nicefrac}       
\usepackage{microtype}      
\usepackage{comment}
\usepackage{cancel}
\usepackage{balance}

\usepackage[dvipsnames]{xcolor}
\usepackage{algorithm}
\usepackage{algorithmic}
\usepackage{amsmath}
\usepackage{amssymb}
\usepackage{amsthm}
\usepackage{graphicx}
\usepackage{dsfont}
\usepackage{url}
\usepackage{xr}
\def\reals{\mathds{R}}

\usepackage{graphicx}
\let\chapter\section
\usepackage{subcaption}
\usepackage{balance}  
\usepackage{footnote}

\usepackage{amsfonts}
\usepackage{amssymb}
\usepackage{amsmath}
\usepackage{latexsym}
\usepackage{verbatim}
\usepackage{url}
\usepackage{calc}
\usepackage{array}

\usepackage{textcomp}
\usepackage{paralist}
\usepackage{comment}
\usepackage{pifont,array,amsmath,multirow,latexsym,tabularx,amssymb}

\usepackage{url}

\usepackage[shortlabels]{enumitem}
\setenumerate{1.,leftmargin=15pt}

\usepackage{graphicx}
\usepackage{xspace}
\usepackage{amsmath}
\usepackage{amssymb}

\usepackage{listings}
\usepackage{graphicx}

\allowdisplaybreaks

\DeclareMathOperator*{\argmin}{arg\,min}
\DeclareMathOperator*{\argmax}{arg\,max}

\newcommand{\norm}[1]{\left\lVert#1\right\rVert}

\usepackage{amsfonts}

\newcommand{\inv}[0]{}
\newcommand{\sinv}[0]{}

\newcommand*{\bmfontA}{\fontfamily{fvs}\selectfont}
\newcommand*{\bmfontB}{\fontfamily{fi4}\selectfont}
\newcommand*{\bmfontC}{\fontfamily{lmss}\selectfont}
\newcommand*{\bmfontD}{\fontfamily{put}\selectfont}
\newcommand*{\bmfontE}{\fontfamily{ppl}\selectfont}
\newcommand*{\bmfontF}{\fontfamily{cmr}\selectfont}
\DeclareTextFontCommand{\bmfont}{\bmfontB}
\DeclareTextFontCommand{\bmfontt}{\bmfontA}
\DeclareTextFontCommand{\bmfonttt}{\bmfontC}
\DeclareTextFontCommand{\bmfontttt}{\bmfontD}
\DeclareTextFontCommand{\bmfonttttt}{\bmfontE}
\DeclareTextFontCommand{\bmfontttttt}{\bmfontF}
\DeclareTextFontCommand{\bmfonttttttt}{\bmfontG}

\newcommand{\ignore}[1]{}
\newcommand{\icmlcut}[1]{}
\newcommand{\nipscut}[1]{}
\newcommand{\nipsdel}[1]{\delete{\nipscut{#1}}}
\newcommand{\cameraReadyAdd}[1]{#1}

\newcommand{\ph}[1]{\noindent \textbf{#1} ---}

\newcommand{\tofix}[1]{\textcolor{red}{#1}}
\newcommand{\tempcut}[1]{}

\newcommand{\barzan}[1]{\textcolor{blue}{BM: #1}}

\definecolor{mypink3}{cmyk}{0, 0.7808, 0.4429, 0.1412}
\newcommand{\junkun}[1]{\textcolor{magenta}{#1}}
\newcommand{\boyu}[1]{\textcolor{orange}{boyu: #1}}
\newcommand{\jarrid}[1]{\textcolor{orange}{Jarrid: #1}}

\newcommand{\fix}[1]{#1}
\newcommand{\delete}[1]{}
\newtheorem{theorem}{Theorem}
\newtheorem{lemma}{Lemma}

\newtheorem{definition}{Definition}

\title{Revisiting Projection-Free Optimization 
           for Strongly Convex Constraint Sets}

\author{
    Jarrid Rector-Brooks \\
    2260 Hayward St \\
    Ann Arbor, MI, 48104 \\
    University of Michigan, Ann Arbor \\
    \texttt{jrectorb@umich.edu}
    \And
    Jun-Kun Wang\thanks{Work performed while a PhD student at the University of Michigan, Ann Arbor.} \\
    226 Ferst Drive NW \\
    Atlanta, GA, 30332 \\
    Georgia Institute of Technology \\
    \texttt{jimwang@gatech.edu}
    \And
    Barzan Mozafari \\
    2260 Hayward St \\
    Ann Arbor, MI, 48104 \\
    University of Michigan, Ann Arbor \\
    \texttt{mozafari@umich.edu}
}

\begin{document}

\maketitle

\begin{abstract}
\icmlcut{FW 
	optimization has attracted much interest in the recent years, due to its scalability  
	as well as the sparsity of its solution.  
While faster at each iteration due to its lack of projection step,
	FW approaches in general 
                 offer a slower convergence rate than (accelerated) gradient descent.}
		
We revisit the Frank-Wolfe (FW) optimization under strongly convex constraint sets. 
	We provide a faster convergence rate for FW without line search, 
		showing that a previously overlooked variant of FW 
			is indeed faster than the standard variant. 
With line search, we show that FW   can converge to the global optimum,
		even for smooth functions that are not convex, but are quasi-convex and locally-Lipschitz.
We also show that, for the general case of (smooth) non-convex  functions, 
    FW with line search converges with high probability to a stationary point at a rate of $O(\frac{1}{t})$, 
        as long as   the constraint set is strongly convex---one of the fastest 
        		convergence rates in   non-convex optimization.

\icmlcut{This new finding has significant ramifications in practice, 
%
			as many classification, regression, multitask learning, and collaborative filtering
				tasks use norm constraints that are strongly convex\ignore{,
                    e.g., generalized linear models with $\fix{l_{2}$ norm}\jarrid{$l_2$-norm -> $l_2$ norm}, 
						squared loss regression with \fix{$l_p$ norm}\jarrid{$l_p$-norm -> $l_p$ norm}, 
							multitask learning with Group Matrix norm, 
							and matrix completion with Schatten norm}.
	We also empirically confirm our theoretical results, showing 
	that Primal Averaging's faster convergence rate 
		(and hence, fewer iterations) 
			more than compensates for its heavier computation at each iteration.} 
\end{abstract}

\section{Introduction}
\label{sec:intro}

\tempcut{
Machine learning has become the de facto approach to predictive analytics on big data.
 Here, a common workflow is to train a model, 
	analyze the results, adjust the parameters or features, and then repeat the process until a desired objective is met.
Solving an optimization problem is often a fundamental building block of the training phase,
	which is also the computational bottleneck.
}

 A popular family of optimization algorithms are so-called gradient descent algorithms:  iterative algorithms that are comprised of a gradient descent step
	at each iteration, followed by a projection step when there is a feasibility constraint.
 The purpose of the projection is to ensure that the update vector remains within the feasible set. 
\ignore{
The projection cost varies depending on the specific constraint. 
For example, 
    projecting onto an $l_1$, $l_2$, or $l_\infty$ norm ball has a closed-form, which can be computed efficiently~\cite{H14}.
}

In many cases, however, the projection step may have no closed-form and thus requires solving another optimization problem itself (e.g., for $l_{1.5}$ norm balls or matroid polytopes~\cite{H14,HK12}), 
     the closed-form may exist but involve an expensive computation
  (e.g., the SVD of the model matrix for Schatten-$1$, Schatten-$2$, and Schatten-$\infty$ norm balls~\cite{H14}),
  or there may simply be no method available for computing the projection in general (e.g., the convex hull of rotation matrices~\cite{HazanRotations}, which arises as a constraint set in online learning settings~\cite{HazanRotations}). 
In these scenarios, 
    each iteration of the gradient descent may require many ``inner'' iterations to compute
	the projection~\cite{J10,FWLinConvPolytope,HK12}.
This makes
	the projection step quite costly, and can account for much of the execution 
time of each iteration (e.g., see Appendix B).  
\tempcut{
Another motivation for FW optimization is that 
	the performance of gradient descent is
    heavily reliant on the step sizes at each iteration, causing  
    	practitioners to spend a significant amount of time finding an optimal sequence
    of step sizes.
}

\ph{Frank-Wolfe (FW) optimization}
 In this paper, we focus on (FW) approaches, also known as \emph{projection-free} or \emph{conditional gradient} algorithms~\cite{frank1956algorithm}.
Unlike gradient descent, these algorithms avoid the projection step altogether by ensuring that the update vector
     always lies within the feasible set.
At each iteration, FW solves   
a linear program over a constraint set.
Since linear  programs  have closed-form solutions for most constraint sets,
		each iteration of FW is, in many cases, more cost effective than
			conducting a gradient descent step and then projecting it back to the constraint set~\cite{J13,HK12,H14}.

Another main advantage of FW is the sparsity of its solution. 
Since the solution of a linear program is always a vertex (i.e., extreme point) of
the feasible set (when the set itself is convex), 
     each iteration of FW can add, at most, one new vertex
     to the solution vector.  Thus, at iteration $t$, the solution is
     	a combination of, at most, $t + 1$ vertices of the feasible set, thereby
    		guaranteeing the sparsity of the eventual solution~\cite{FWSparsityWithCoresets,J13,JaggiSparsityThesis}.

For these reasons, FW optimization has drawn growing interest in recent years, 
especially in matrix completion, structural SVM, 
		computer vision, sparse PCA, metric learning, and many other settings~\cite{J10,SJM13,O16,W16,CVPR15,ImgClassTraceNorm,HK12,LargeScaleConvMinLowRankShalev}.
Unfortunately, while faster in each iteration, 
	standard FW requires many more iterations to converge than gradient descent, 
		and therefore is slower overall. 
        This is because FW's convergence rate is typically $O\left(\frac{1}{t}\right)$ while that 
        of (accelerated) gradient descent is $O\left(\frac{1}{t^2}\right)$, where $t$
        is the number of iterations~\cite{J13}.

\nipscut{
	One    known strategy to improve FW's performance \jarrid{when the constraint set is strongly convex} is
	to include a line search at each iteration~\cite{D15},
		which achieves a convergence rate of $O\left(\frac{1}{t^2}\right)$
            at the cost of making each iteration significantly slower. 
Interestingly, FW without line search is much faster in practice than FW with line search, 
	as the additional cost of   line search at each iteration can outweigh
		the faster convergence rate (see our technical report~\cite{fw_tr}).
}


\begin{table*}[t]\centering
\scalebox{0.76}{
\begin{tabular}{@{}p{4.3cm}p{4.05cm}p{2.3cm}p{2.7cm}p{3.2cm}@{}}\toprule
    & Additional Assumptions \newline about the Loss Function & Constraint Set \newline Assumption & Convergence \newline Rate & Requires Line Search  \newline (In Each Iteration)\\ \midrule
    
    \textbf{Convex Loss Function} \\    
    \hspace{2mm}This Paper & \delete{Optimal points are on the boundary of the constraint set} \fix{None} & Strongly convex & $O\left( \frac{1}{t^2} \right)$ with \newline high probability& No \\ 
    \hspace{2mm}State-of-the-Art Result(s)\\
    \hspace{6mm}\cite{J13} & None & Convex & $O\left( \frac{1}{t} \right)$ & No \\
    \hspace{6mm}\cite{D15} & Strongly convex & Strongly convex & $O\left( \frac{1}{t^2} \right)$ & \textbf{Yes} \\
    \hspace{6mm}\cite{FWLinConvPolytope} & Strongly convex & Polytope & $O\left( \exp\left( -t \right) \right)$ & \textbf{Yes} \\
    \hspace{6mm}\cite{LP66,DR70,D79} & Norm of the gradient \newline is lower bounded &  Strongly convex & $O\left( \exp\left( -t \right) \right)$& No \\ 
    \hspace{6mm}\cite{BeckLinConverge} & $f(x) = \norm{Ax - b}_2^2$ & Convex & $O\left( \exp\left( -t \right) \right)$ & No \\
    \midrule

    \textbf{Quasi-Convex Loss Function} \\
    \hspace{2mm}This Paper & Locally-Lipschitz, \newline \fix{Norm of the gradient \newline is lower bounded} & Strongly convex & $O\left( \min\left(\frac{1}{t^{1 / 3}}, \frac{1}{t^{1 / 2}}\right) \right)$ & \textbf{Yes} \\
    \hspace{2mm}State-of-the-Art Result(s)\\
    \hspace{6mm}Does not exist & Does not exist & Does not exist & Does not exist & Does not exist \\
    \midrule

    \textbf{Non-Convex Loss Function} \\
    \hspace{2mm}This Paper & \delete{Norm of the gradient \newline is lower bounded} \fix{None} & Strongly convex & $O\left( \frac{1}{t} \right)$ with \newline high probability& \textbf{Yes} \\
    \hspace{2mm}State-of-the-Art Result(s)\\
    \hspace{6mm}\cite{S16} & None & Convex & $O\left( \frac{1}{t^{1 / 2}} \right)$ & No \\
\bottomrule
\end{tabular}
}
\parbox{15.5cm}{
\caption{
    Our contributions compared to the state-of-the-art results for projection-free optimization.
    Here, $t$ is the number of iterations.     
    \ignore{, $n$ is the dataset size (number of tuples in the training set), 
    and $c$ is the number of iterations required by a line search (used in~\cite{D15},  
    which can drastically increase optimization time for large datasets)}
	For non-convex functions, convergence is defined 
    in terms of a stationary point instead of a global minimum.
    Note that although our bound is probabilistic for convex loss functions,
    we use no additional assumptions on the loss function
    and do \emph{not} require line search, which can be a costly operation  for big data (see Section~\ref{sec:related}).
}	
\label{tab:summary}
}
\end{table*}

We
make several contributions (summarized in Table~\ref{tab:summary}):
\begin{enumerate} 
\item 
 We revisit a non-conventional variant of FW optimization,
called Primal Averaging (PA)~\cite{L13}, 
		which has been largely neglected in the past, as 
			it was believed to have the same convergence rate as FW without line search,
				yet incurring extra computations (i.e., matrix averaging step)
				at each iteration.
However, we discover that, when the constraint set is strongly convex,
    this non-conventional variant  enjoys a much faster convergence rate with high probability, 
        $O(\frac{1}{t^2})$ versus $O(\frac{1}{t})$,
		which more than compensates for its slightly more expensive iterations.
        This surprising result has important ramifications in practice, 
 			as many classification, regression, multitask learning, and collaborative filtering
				tasks rely on norm constraints that are strongly convex,
                e.g., generalized linear models with $l_{p}$ norm, 
						squared loss regression with $l_p$ norm, 
							multitask learning with Group Matrix norm, 
							and matrix completion with Schatten norm~\cite{kim2010tree,D15,H14}.

\item 
    While previous work on FW optimization has generally focused on 
	convex functions, we show that FW with line search can converge to the global optimum,
		even for smooth functions that are not convex, but are quasi-convex and locally-Lipschitz.

\item We also study the general case of (smooth) non-convex  functions, 
		showing that FW with line search can converge to a stationary point
        at a rate of $O(\frac{1}{t})$ with high probability,
        as long as  the constraint set is strongly convex.
	To the best of our knowledge, we are not aware of such a fast convergence rate in the non-convex optimization literature.\footnote{Without any assumptions, converging to   local optima for continuous non-convex functions  is NP-hard \cite{GDHS17,AABHM17}.}

\item Finally, we conduct extensive experiments on various benchmark datasets, 
		empirically validating our theoretical results, 
		and comparing the actual performance of various FW variants in practice.

\end{enumerate}


\icmlcut{The rest of this paper is organized as follows. We discuss the related work in Section~\ref{sec:related}, 
	and briefly review the background material in Section~\ref{sec:background}.
We describe the Primal Averaging algorithm and our results for smooth convex functions in Section~\ref{sec:convex}. 
We present our results    
 for smooth quasi-convex functions in Section~\ref{sec:quasi}, and for non-convex functions in Section~\ref{sec:nonconvex}.
Finally, we conduct an extensive set of experiments in Section~\ref{sec:expr}, and conclude in Section~\ref{sec:conclusion}.}

\section{Related Work}
\label{sec:related}

 Table~\ref{tab:summary}  compares the state-of-the-art on projection-free optimization to our contributions.

\ph{Convex optimization}
Garber and Hazan~\cite{D15} show that for strongly convex and smooth loss functions, FW with line search achieves a convergence rate
of $O(\frac{1}{t^2})$ over strongly convex sets. In contrast, we do not need the loss function to be strongly convex.
     Further, they
    	require an exact line search  at each iteration to achieve this convergence 
    rate. Line search, however,    comes with significant downsides. 
     An exact line search solves the problem $\underset{\gamma \in [0, 1]}{\min} f(x + \gamma v)$
    for loss function $f$, solution vector 
    $x \in \reals^n$, and descent direction $v \in \reals^n$.  
    There are several methods for solving this optimization, and 
    choosing the best 
    method is often difficult for   practitioners (e.g., bracketing line searches versus 
    	interpolation ones).
 \icmlcut{For example, bracketing line
    searches (e.g., Golden Section or Fibonacci methods) perform well for functions
    that are non-smooth or have complicated derivatives.  Contrarily, interpolation line searches
    (e.g., Quadratic or Cubic Interpolation) are better if the loss function's derivatives are easy to compute.}
    Moreover, at best,
     these methods converge to the minimum at a rate of 
    $O\left( \frac{1}{t^2} \right)$~\cite{LineSearchSource}.
Approximate line searches   require fewer iterations.  
    However, in using them, one loses most theoretical guarantees provided in previous work, 
    including that of~\cite{D15}. 
    Nonetheless, both exact and inexact line searches involve
    at least one evaluation of the loss function or one of its derivatives, 
	which can be quite prohibitive for large datasets  (see Section~\ref{sec:expr:times}).
    This is because the underlying function for data modeling is typically in the form of 
    a finite sum (e.g., regression loss) over all the data. 
    In comparison, Primal Averaging, which we study and promote,
    does not require a line search and works with a predefined step size.
    Notably, this allows PA to considerably outperform FW with line search 
    (see Section~\ref{sec:expr:times}).

\ignore{The minimum of the new polynomial
    is compared to previously known values of $f$, and their similarity is used to reduce
    the uncertainty interval.  When the quantity $|b - a| \leq \delta$ for some small $\delta$,
    the points $a$ and $b$ are viewed as minimizers and returned.}

\ignore{\footnote{Bracketing line searches find an optimal step size by iteratively
    considering smaller}
}

Prior work~\cite{LP66,DR70,D79}  shows that standard FW without line search for smooth functions can achieve an exponential 
convergence rate, by making a strict assumption that the gradient is lower-bounded everywhere in the feasible set. 
In our analysis of PA, however, we do not assume the gradient is lower-bounded everywhere, allowing our result to be more widely
applicable.


\ph{Quasi-convex optimization}
Hazan et al. study quasi-convex and locally-Lipschitz loss functions that admit some saddle points~\cite{HLS15}.
One of the optimization algorithms 
 for this class of functions is the so-called \emph{normalized gradient descent}, 
		which converges to an $\epsilon$-neighborhood of the global minimum. 
		The analysis in \cite{HLS15} is for unconstrained optimization. 
 In this paper, we analyze FW for the same class of functions, but with strongly convex constraint sets. 
Interestingly, when the constraint set is an 
$l_2$ 
ball, FW becomes equivalent to normalized gradient descent.
    In this paper, we both 1) show that FW can converge to a neighborhood
    of a global minimum, and 2) derive a convergence rate.
    \ignore{Not only do we show that FW  can converge to a neighborhood of a global minimum, but we also derive a
          convergence rate.}
 \cite{D79} extends the analysis of FW to a class of quasi-convex functions of the form 
 $f(w) := g(h(w))$, where $h$ is   differentiable and monotonically increasing, and $g$ is a smooth function. 
 Such functions are quite rare in machine learning.  In contrast, 
we study a much more general class of quasi-convex functions, including several popular models (e.g., 
  generalized linear models with a sigmoid loss). 
  
 \ph{Non-convex optimization} While there has been a surge of research on non-convex 
 optimization in recent years~\cite{GDHS17,GHJY15,AABHM17,LSJR16,S16},  
nearly all of it has focused on unconstrained optimization. 
To our knowledge, there are only a few exceptions~\cite{S16,GL16,GHJY15,RSPS16}. 
    \ignore{Lacoste-Julien~}\cite{S16} proves that FW for smooth non-convex  functions 
	converges to a stationary point, at 
    a rate of $O(\frac{1}{\sqrt{t}})$, which matches the rate of projected gradient descent. 
     \cite{RSPS16} extends this and considers a stochastic version of FW for smooth non-convex functions. Furthermore, Theorem 7 of~\cite{YZS14} provides a convergence rate for non-convex  optimization using FW, 
     which is slower than $O(\frac{1}{\sqrt{t}})$. 
     We show in this paper that, for strongly convex sets,
FW converges to a stationary point with high probability much faster: $O(\frac{1}{t})$.

\begin{algorithm*}[t] 
   \begin{algorithmic}[1]
   \caption{Standard Frank-Wolfe algorithm}\label{alg:fw}
\STATE Input: loss $f : \Omega \to \reals$.
\STATE Input: linear opt. oracle $\mathcal{O}(\cdot)$ for $\Omega$.
\STATE Initialize: any $w_{1} \in \Omega$.
\FOR{$t=1, 2, 3, \dots$}
\STATE $ v_t \leftarrow \mathcal{O}(\nabla f(w_{t}) ) = \arg\min_{v \in \Omega} \langle v, \nabla f(w_{t})  \rangle $.
\STATE Option (A): Predefined decay learning rate $\{ \gamma_t \in [0,1] \}_{t=1, 2, \ldots}$
\STATE Option (B): $\gamma_t$$=$$\arg\min_{\gamma \in [0,1]}
 \gamma \langle v_t - w_{t} , \nabla f(w_{t}) \rangle + \gamma^2 \frac{L}{2} \| v_t - w_{t} \|^2  $.
\STATE $w_{t+1} \leftarrow (1 - \gamma_t) w_{t} + \gamma_t v_t$. 
\ENDFOR
\end{algorithmic}
\end{algorithm*}

\begin{algorithm*}[t] 
\begin{algorithmic}[1] 
   \caption{Primal Averaging} \label{alg:Nest2}
\STATE Initialize any $v_{0} \in \Omega \subset \reals^d$. Set $w_0 = v_{0}$.
\FOR{$t=1, 2, 3, \dots$}
\STATE \quad  $\gamma_t = \frac{2}{t+1}$.
\STATE \quad  $z_{t-1} = (1 - \gamma_t) w_{t-1} + \gamma_t v_{t-1}$.
\STATE \quad  Option (A): $p_t = \Sigma_{i=1}^t \frac{\theta_i}{\Theta_t}  \nabla f(z_{i-1})$, where $\Theta_t = \Sigma_{i=1}^t \theta_i$, \fix{$\theta_t = t$,} and $\frac{\theta_t}{ \Theta_t }= \gamma_t$.
\STATE \quad  Option (B): $p_t = \nabla f( z_{t-1} )$.
\STATE \quad  $v_{t} = \underset{ v \in \Omega}{ \arg\min} \langle v, p_t \rangle $.
\STATE \quad  $w_{t} = (1 - \gamma_t) w_{t-1} + \gamma_t v_t$.
\ENDFOR
\end{algorithmic}
\end{algorithm*}

\section{Background}
\label{sec:background}

\subsection{Preliminaries}
\label{sec:prelim}

\icmlcut{
\delete{
We first review the definition of convexity and strong convexity.
}

\jarrid{Delete definition of strongly convex sets.}
}

Strongly convex constraint sets are quite common in  machine learning.
For example, when ${p \in (1, 2]}$,
	$\mathbf{l_p}$ balls  $\{ u \in \reals^n : \norm{u}_p \leq r\}$
	and Schatten-$p$ balls $\{X \in \reals^{m \times n} : \norm{X}_{\mathds{S}_p} \leq r\}$
	are all strongly convex \cite{D15}, 
            where 
            $\norm{X}_{\mathds{S}_p} = \left(\sum_{i = 1}^{\min(m, n)} \sigma(X)_i^p\right)^{1 / p}$         
            is the Schatten-$p$ norm and $\sigma(X)_i$ is the $i^{th}$ largest singular value of $X$. Group $l_{p, q}$ balls, used in multitask learning~\cite{D15,kim2010tree}, are also strongly convex when ${p, q \in (1, 2]}$.
\icmlcut{\begin{definition}[\textbf{Smooth functions}]
\delete{
A function $f: \reals^d \rightarrow \reals$ is $L$-smooth with respect to 
a norm $\norm{\cdot}$, if for every $u, v \in \reals^d$, $f$ satisfies
}
\[ | f( u ) - f(v) - \nabla f(v)^\top ( u - v ) | \leq \frac{L}{2} \norm{u - v}^2 \]
\end{definition}
}
In this paper, we use the following definitions.

\cameraReadyAdd{
\begin{definition}[\textbf{Strongly convex set}]
 \label{def:strongConvSet}
 A convex set $\Omega \subseteq \reals^d$ is an 
\emph{$\alpha$-strongly convex set} with respect to a norm $\| \cdot \|$
if for any $u, v \in \Omega$ and any $\theta \in [0,1]$,
the ball induced by $\| \cdot \|$ which is centered at 
$\theta u + ( 1 - \theta) v$ with radius 
$\theta (1 - \theta) \frac{\alpha}{2} \| u - v \|^2$ is also included in $\Omega$.
\end{definition}
}

\begin{definition}[\textbf{Quasi-convex functions}]
    A function \newline ${f: \reals^d \rightarrow \reals}$ is quasi-convex if for all ${u, v \in \reals^d}$
    such that ${f(u) \leq f(v)}$, it follows that ${\langle \nabla f(v), u - v \rangle \leq 0}$,
    where ${\langle \cdot, \cdot \rangle}$ is the standard inner product.
\end{definition}

\begin{definition}[\textbf{Strictly-quasi-convex functions}]
    A function ${f: \reals^d \rightarrow \reals}$ is strictly-quasi-convex if it is
    quasi-convex and its gradients only vanish at the global minimum.  That is, for all
    $u \in \reals^d$, it follows that $f(u) > f(u^*) \Rightarrow \norm{\nabla f(u)} \neq 0$
    where $u^*$ is the global minimum.
\end{definition}

\ignore{
\begin{definition}[\textbf{Locally-Lipschitz functions}]
\nipsdel{
    Let     $\mathds{B}_r(x)$ denote the Euclidean norm ball centered at $x$ of radius $r$
    where $x \in \reals^d$ and $r \in \reals$. 
  Furthermore, let $G, \epsilon \geq 0$.
    Then for arbitrary set $\mathcal{K}$, define ${f: \mathcal{K} \rightarrow \reals}$.  We say
    $f$ is $(G, \epsilon, x)$-locally-Lipschitz, if for each $u, v \in \mathds{B}_\epsilon(x)$
    the following holds:
}
    \[ \left| f(u) - f(v) \right| \leq G \norm{u - v} \]
\end{definition}
}

\begin{definition}[\textbf{Strictly-locally-quasi-convex functions}]
    Let $u, v \in \reals^d, \kappa, \epsilon > 0$.  
    \fix{Further, write $\mathds{B}_r(x)$ as the Euclidean norm ball 
    centered at $x$ of radius $r$ where $x \in \reals^d$ and $r \in \reals$.}
    We say
    ${f: \reals^d \rightarrow \reals}$ is $(\epsilon, \kappa, v)$-strictly-locally-quasi-convex
    in $u$ if at least one of the following applies:
    \begin{enumerate}
        \item
            $f(u) - f(v) \leq \epsilon$

        \item
            $\norm{\nabla f(u)} > 0$ and for every $y \in \mathds{B}_{\frac{\epsilon}{\kappa}}(v)$ 
            it holds that ${\langle \nabla f(u), y - u \rangle \leq 0}$
    \end{enumerate}
\end{definition}

\tempcut{Note that if $f$ is both strictly-quasi-convex and locally-Lipschitz, 
  it is also  strictly-locally-quasi-convex~\cite{HLS15}.}  
  
\subsection{A Brief Overview of Frank-Wolfe (FW)}
\label{sec:fw}

The Frank-Wolfe (FW) algorithm (Algorithm~\ref{alg:fw}) attempts to solve the 
constrained optimization problem $\underset{x \in \Omega}{\min} f(x)$
for some convex constraint set $\Omega$ (a.k.a. feasible set) and some function $f: \Omega \rightarrow \reals$. 
FW   begins with an initial solution $w_0 \in \Omega$.
Then, at each iteration, it computes a search direction $v_t$ by
minimizing the linear approximation of $f$ at $w_t$, $v_t = \underset{v \in \Omega}{\min} \langle v, \nabla f(w_t) \rangle$,
where $\nabla f(w_t)$ is the gradient of $f$ at $w_t$.
 Next, FW produces a convex combination of the current iterate $w_t$ and 
the search direction $v_t$ to find the next iterate $w_{t + 1} = (1 - \gamma_t) w_t + \gamma_t v_t$
where $\gamma_t \in [0, 1]$ is the learning rate for the current iteration. 
\tempcut{Note that $w_{t + 1}$ is always in $\Omega$ since $w_t, v_t \in \Omega$
and $\Omega$ is a convex set.}\ignore{
    For $x, y \in U$ where $U$ is a convex set, the convex
    combination $(1 - \theta)x + \theta y$ is also in $U$
    for $\theta \in [0, 1]$.
}
There are a number of ways to choose the learning rate $\gamma_t$.  
Chief among these are setting \fix{$\gamma_t = \frac{2}{t + 1}$} 
(Algorithm~\ref{alg:fw}, option A) or finding $\gamma_t$ via  line search (Algorithm~\ref{alg:fw}, option B).

 \section{Faster Convergence Rate for Smooth Convex Functions}
\label{sec:convex}
  
  \inv
\subsection{Primal Averaging (PA)}

PA \cite{L13} (Algorithm~\ref{alg:Nest2}) is a variant of FW that operates in a style similar to Nesterov's
acceleration method.    PA maintains three sequences,
$(z_{t-1})_{t=1,2,\dots}$, $(v_t)_{t=1,2,\dots}$, and $(w_{t})_{t=1,2,\dots}$.
The first is the accelerating sequence (as in Nesterov acceleration),
the second is the sequence of search directions, and the third is the sequence of solution
vectors.  At each iteration, PA updates its sequences by computing two
convex combinations and consulting the linear oracle, such that
\[ z_{t - 1} = (1 - \gamma_t) w_{t - 1} + \gamma_t v_{t - 1} \]
\[ v_t = \underset{v \in \Omega}{\argmin} \langle \Theta_t^{-1} \sum_{i = 1}^{t} \theta_i \nabla f(z_{i - 1}), v \rangle \]
\[ w_t = (1 - \gamma_t) w_{t - 1} + \gamma_t v_t \]
where $\Theta_t = \sum_{i = 1}^t \theta_i$ and the $\theta_i$ are chosen, such that $\gamma_t = \frac{\theta_t}{\Theta_t}$.
\fix{
    Note that choosing $\theta_t$ does not require significant computation as 
    setting $\theta_t = t$ satisfies the requirement $\gamma_t = \frac{\theta_t}{\Theta_t}$ for all $t$.
}
\footnote{
    \fix{
        If $\theta_t = t$ then
        $\frac{\theta_t}{\Theta_t} = \frac{t}{\sum_{i = 1}^t i} = \frac{2 t}{t (t + 1)} = \frac{2}{t + 1} = \gamma_t$.
    }
}

Since $z_{t - 1}$ and $w_t$ are convex combinations of elements of the constraint
set $\Omega$, $z_{t-1}$ and $w_t$ are themselves in $\Omega$.
While the input to the linear oracle is a single gradient vector in standard FW, PA uses 
an average of the gradients seen in iterations $1, 2, \dots, t$ as the input to the linear oracle.

In standard FW,  the sequence $( w_t )_{t = 1, 2, \dots}$   has 
the following property \cite{J13,L13,H14}:
\begin{equation} \label{eq:vws}
    f(w_t) - f(w^*) \leq \frac{2L}{t(t+1)} \Sigma_{i=1}^t \norm{v_i - w_{i - 1}}^2
\end{equation}
where $w^*$ is an optimal point and $L$ is the smoothness parameter of $f$.
We observe that the $\frac{1}{t} \sum_{i = 1}^t \norm{v_i - w_{i - 1}}$ factor of (\ref{eq:vws}) is the
average distance between the search direction and solution vector pairs.
Denote the diameter $D$ of $\Omega$ as $D = \underset{u, v \in \Omega}{\sup} \norm{u - v}$.
Then, since $w_{i - 1}$ and $v_i$ are both in $\Omega$, we find that $\frac{1}{t} \sum_{i = 1}^t \norm{v_i - w_{i - 1}} \leq D$.
That is, the average distance of $v_i$ and $w_{i - 1}$ is upper bounded by  
diameter $D$ of $\Omega$.  Combining this   with (\ref{eq:vws}) yields standard FW's 
convergence rate:
 \begin{equation} \label{eq:SFWConvRate}
    \begin{aligned}
        f(w_t) - f(w^*) &\leq \frac{2L}{t(t+1)} \Sigma_{i=1}^t \norm{v_i - w_{i - 1}}^2 \\
                        &\leq \frac{2LD^2}{t + 1} = O\left(\frac{1}{t}\right)
    \end{aligned}
\end{equation}
 PA has a similar guarantee for the sequence $( w_t )_{t = 1, 2, \dots}$~\cite{L13}.  Namely
\begin{equation} \label{eq:vvs}
    f(w_t) - f(w^*) \leq \frac{2L}{t(t+1)} \Sigma_{i=1}^t \| v_i - v_{i-1} \|^2 
\end{equation}
While the inability to guarantee an arbitrarily small distance between $v_i$ and $w_i$ in
Equation~\ref{eq:vws} caused standard FW to converge as $O( \frac{1}{t} )$, 
this is not the case for the distance between $v_i$ and $v_{i - 1}$ in Equation~\ref{eq:vvs}.
Should we be able to bound the distance $\norm{v_i - v_{i - 1}}$ to be arbitrarily small,
we can show that PA converges as $O( \frac{1}{t^2} )$ with high probability.  
We observe that the sequence $(v_t)_{t = 1, 2, \dots}$ expresses this behavior when
the constraint set is strongly convex. We have the following theorem.\footnote{All omitted proofs can be found in Appendix A.}

\begin{theorem} \label{thm:convex}
    Assume the \cameraReadyAdd{convex} function $f$ is smooth with parameter $L$. \fix{Further, define 
    the function $h$ as $h(w) = f(w) + \theta \xi^T w$ where 
    \cameraReadyAdd{$\theta \in \left(0, \frac{\epsilon}{4D} \right]$},
$\xi \in \reals^d$, $w \in \Omega$, $\Omega$ is an $\alpha$-strongly convex set,
\cameraReadyAdd{$D$ is the diameter of $\Omega$,}
and $\xi$ is uniform on the unit sphere.}  Applying 
PA to \fix{$h$} \delete{over $\Omega$ with 
dimensionality $d$} yields the following convergence rate \fix{for $f$} with probability $1 - \delta$,
\[f(w_t) - f(w^*) = O\left( \frac{d L}{\alpha^2 \delta^2 t^2 } \right)\]
\end{theorem}

\fix{Theorem~\ref{thm:convex} states that applying PA to a perturbed function $h$
over an $\alpha$-strongly convex constraint set
allows any smooth, convex function $f$ to converge as $O\left( \frac{1}{t^2}\right)$}
\delete{Theorem~\ref{thm:convex} states that PA has an $O\left( \frac{1}{t^2}\right)$ 
    convergence rate for smooth functions }
    with probability $1 - \delta$, albeit depending on   $\delta$ and $d$.  
    However, as $t$ grows, the $t^2$ term in the convergence rate's denominator quickly   dominates
    the rate's $\delta$ and $d$ terms. This, combined with PA's non-reliance on line 
    search, allows it to outperform the method proposed in~\cite{D15}.  
    \fix{
        We note that, although Theorem~\ref{thm:convex} requires us to run PA on 
        the perturbed function $h$, \emph{$f$ itself still converges as 
        $O\left( \frac{1}{t^2}\right)$ with high probability}.
        That is, the iterates $w_t$ produced by running PA on $h$ themselves 
        have the guarantee of 
        $f(w_t) - f(w^*) = O\left( \frac{d L}{\alpha^2 \delta^2 t^2 } \right)$
        for $w^* = \underset{w \in \Omega}{\argmin} f(w)$
        with probability $1 - \delta$.
    }
    We \fix{also} empirically investigate this \fix{result} in Section~\ref{sec:expr}.

\subsection{Stochastic Primal Averaging (SPA)}
\label{sec:spa}

Here we provide a stochastic version of Primal Averaging.  While in the previous section we studied
PA with Option (A) of Algorithm~\ref{alg:Nest2}, we now consider PA with
Option (B) of Algorithm~\ref{alg:Nest2}, providing an analysis of its stochastic version.
That is, $p_t= \tilde{\nabla} f(z_{t-1})$, where 
\fix{
    $\tilde{\nabla} f$ represents the aggregated stochastic gradient constructed as 
    $\tilde{\nabla} f(z_{t - 1}) = \sum_{i \in S_t} \hat{\nabla} f_i (z_{t - 1})$. Further,
    $\hat{\nabla} f_i (\cdot)$ is the stochastic gradient computed 
    with the $i$th item of a dataset of size $N$, while $S_t$ is the set of 
    indices sampled without replacement from $\{1,2,\dots,N\}$ at 
    iteration $t$.  We note that $\left|S_t\right| = \min(t^4, N)$.
}
\begin{theorem}
\label{thm:spa}
\cameraReadyAdd{Assume the convex function $f$ is smooth with parameter $L$.}
Denote $\sigma$ as the variance of a stochastic gradient. Suppose 
$p_t = \tilde{\nabla} f( z_{t-1} )$ and the number of samples used to obtain $p_t$ is $n_t=O(t^4)$.
\fix{Further, define 
the function $h$ as $h(w) = f(w) + \theta \xi^T w$ where 
\cameraReadyAdd{$\theta \in \left(0, \frac{\epsilon}{4D} \right]$},
$\xi \in \reals^d$, $w \in \Omega$, $\Omega$ is an $\alpha$-strongly convex set,
\cameraReadyAdd{$D$ is the diameter of $\Omega$,}
and $\xi$ is uniform on the unit sphere.}
\delete{Then running PA on $h$ yields with probability $1 - \delta$ the convergence rate for $f$}
\fix{
    Then applying PA to $h$ yields the following convergence rate for $f$ 
    with probability $1 - \delta$,
}
$$E[ f(w_t) ] - f(w^*) = O\left( \frac{ d L^2 (D^2 + \sigma) \log t }{\alpha^2 \delta^2 t^2} \right)$$
\end{theorem}
Theorem~\ref{thm:spa} states that the stochastic version of PA maintains an $O\left(\frac{\log t}{t^2}\right)$ 
convergence rate with high probability\fix{, using $h$ in a manner similar to Theorem~\ref{thm:convex}}.
Note that $n_t$ grows as $O(t^4)$ until it begins to use all the data points to compute the gradient. 
 Thus, for earlier iterations of SPA, the algorithm requires far less computation
than its deterministic counterpart.  However, the   samples required in each iteration grows quickly,
causing  later iterations of SPA to share the same computational cost as deterministic Primal
Averaging.
 
\section{Strictly-Locally-Quasi-Convex Functions}
\label{sec:quasi}

In this section we show that FW with line search can converge within an $\epsilon$-neighborhood of 
the global minimum for strictly-locally-quasi-convex functions.  Furthermore, \fix{if} it is assumed that the 
norm of the gradient is lower bounded, then FW with line search can converge within an $\epsilon$-neighborhood
of the global minimum in $O\left( max\left( \frac{1}{\epsilon^2}, \frac{1}{\epsilon^3}\right) \right)$ iterations.

\begin{theorem} \label{thm:quasi}
Assume that the function $f$ is smooth with parameter $L$, and that 
$f$ is $(\epsilon,\kappa,w^*)$-strictly-locally-quasi-convex, where $w^*$ is a global minimum.
Then, the standard FW algorithm with line search (Algorithm~\ref{alg:fw} option (B)) 
can converge within an $\epsilon$-neighborhood of the global minimum when the constraint set 
is strongly convex.
Furthermore, if one assumes that $f(w) - f(w^*) \geq \epsilon$ implies that the norm of the gradient is lower bounded
\fix{as $\| \nabla f (w) \| \geq \theta \epsilon$ {} for some $\theta \in \reals$}, 
then the algorithm needs $t= O(max( \frac{\fix{2 \kappa}}{\fix{\theta} \epsilon^2}, \frac{\fix{8 L \kappa}}{\fix{\theta} \epsilon^3}) )$ iterations 
to produce an iterate that is within an $\epsilon-$neighborhood of the global minimum.
\end{theorem}

\fix{
    Hazan et al.~\cite{HLS15} provide several examples
    of strictly-locally-quasi-convex functions. 
    First, if $\epsilon \in (0, 1]$ and $x = (x_1, x_2) \in [-10, 10]^2$, then the function 
    \[
        g(x) = (1 + e^{-x_1})^{-1} + (1 + e^{-x_2})^{-1}
    \]
    is $(\epsilon, 1, x^*)$-strictly-locally-quasi-convex in $x$.
    Second, if $\epsilon \in (0, 1)$ and $w \in \reals^d$, then the function 
    \[
        h(w) = \frac{1}{m} \sum_{i = 1}^m (y_i - \phi (\langle w, x_i \rangle))^2 
    \]
    is $(\epsilon, \frac{2}{\gamma}, w^*)$-strictly-locally-quasi-convex in $w$.  Here,
    $\phi(z) = \mathds{1}_{z \geq 0}$, 
    $\gamma \in \reals$ is the margin of a perceptron, and we have $m$ samples
    $\{ (x_i, y_i) \}_{i = 1}^m \in \mathds{B}_1 (0) \times \{0, 1\}$ where 
    $\mathds{B}_1 (0) \subset \reals^d$.
}

\delete{
Hazan et al.~\cite{HLS15} provide an example 
of a  strictly-locally-quasi-convex loss in practice: 
 their
model is \delete{$(\epsilon, \exp(W), w^*)$-}strictly-locally-quasi-convex in $w$ for $\epsilon, W \in \reals$, 
$\epsilon > 0$, and $w \in \mathds{B}_W(0)$, where $\exp(\cdot)$ is the exponential
function.  Further, denote the sigmoid function 
$\phi: \reals \rightarrow \reals$ as ${\phi(z)= (1+\exp(-z))^{-1}}$.  Then the generalized
linear model with the sigmoid loss is given by \fix{the smooth function} 
$f(w) = \frac{1}{m} \sum^{m}_{i=1} ( y_i - \phi( \langle w, x \rangle))^2$
where there exists a $w^*$ such that ${y_i = \phi( \langle w^*, x_i \rangle) }$, $\forall i$ and ${\|w^*\|\leq W}$. 
\delete{
    The function $f$ is smooth and it can further be shown that the assumption 
    $f(w) - f(w^*) \geq \epsilon \implies \| \nabla f(w) \| \geq \theta \epsilon$ 
    holds for ${\theta = \frac{3 \exp(-W)}{4 W }}$.
}
}

\begin{table*}[!htbp]
\centering
\scalebox{0.8}{
\begin{tabular}{|c|c|c|c|}
\hline
\textbf{Convexity of Loss Function} & \textbf{Loss Function} & \textbf{Constraint} & \textbf{Task}  \\ \hline
\multirow{2}{*}{Convex} 
& Quadratic Loss & $l_p$ norm & Regression \\ \cline{2-4}
 & Observed Quadratic Loss & Schatten-$p$ norm  & Matrix Completion \\ \hline
Strictly-Locally-Quasi-Convex & Squared Sigmoid & $l_p$ norm & Classification \\ \hline
Non-Convex & Bi-Weight Loss & $l_p$ norm & Robust Regression\\ \hline
\end{tabular}
}
\caption{Various loss  functions and constraint sets used in our experiments.}
\label{tbl:obj_func}
\end{table*}
 
\begin{figure*}[!ht]
\inv
\centering     
\begin{subfigure}[t]{0.35\linewidth}
\centering
    \includegraphics[height=50mm, keepaspectratio]{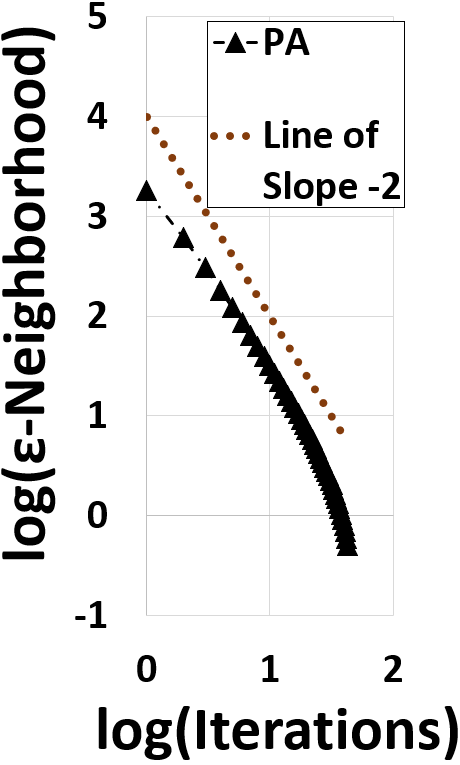}
    \centering
    \caption{Matrix completion w/ convex (observed quadratic) loss, Schatten-$2$ norm constraint.}
    \label{fig:MtxCmpPAConvRate}
\end{subfigure}
\hspace*{0.3cm}
\begin{subfigure}[t]{0.23\linewidth}
\centering
    \includegraphics[height=50mm, keepaspectratio]{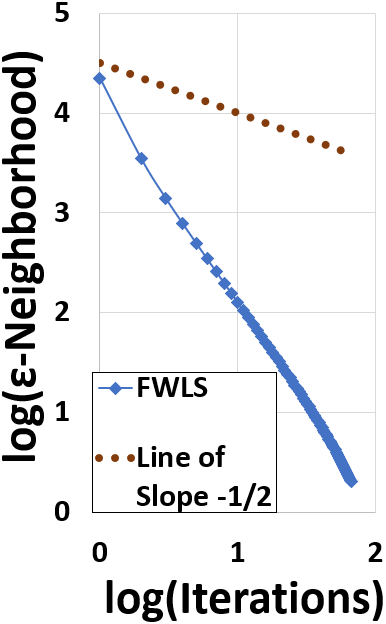}
    \caption{Classification w/ quasi-convex (squared sigmoid) loss, $l_2$ norm constraint.} 
    \label{fig:L2PAConvRate}
\end{subfigure}
\hspace{0.3cm}
\begin{subfigure}[t]{0.3\linewidth}
\centering
    \includegraphics[height=50mm, keepaspectratio]{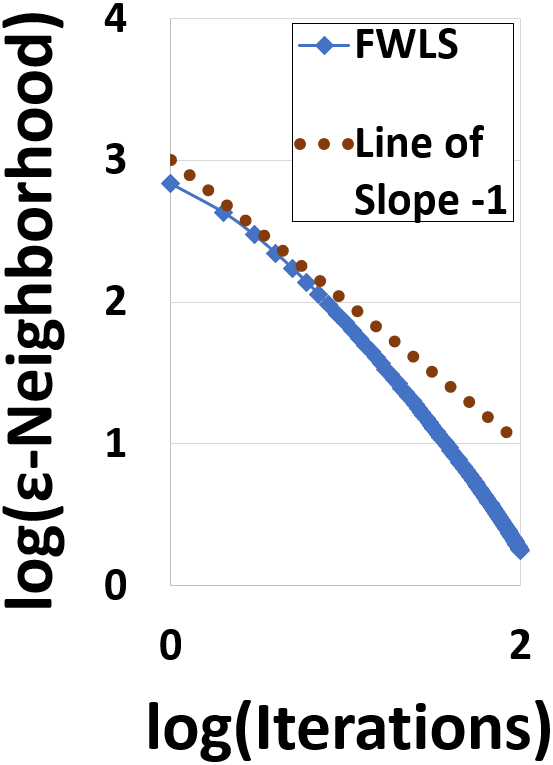}
    \centering
    \caption{Regression w/ non-convex (bi-weight) loss, $l_2$ norm constraint.}
    \label{fig:FWLSNonConvex}
\end{subfigure}
\inv
\caption{Convergence rates of FW variants for convex loss without line search and non-convex loss with line search.} 
\label{fig:ConvRates}
\end{figure*}

\section{Smooth Non-Convex Functions}
 \label{sec:nonconvex}

In this section, we show that, \fix{with high probability}, FW with line search
converges as $O\left( \frac{1}{t} \right)$ to a stationary point when
the loss function is non-convex and the constraint set is strongly convex.
To our knowledge, a rate this rapid does not exist in the non-convex optimization literature.

To help demonstrate our theoretical guarantee, we introduce a measure called the
FW gap. The FW gap of $f$ at a point $w_t \in \Omega$ is defined as $k_t := \max_{ v \in \Omega } \langle v - w_t, - \nabla f( w_t) \rangle$.
This measure is adopted in~\cite{S16}, which is the first work to show that, for smooth 
non-convex functions, FW has an $O\left( \frac{1}{\sqrt{t}} \right)$ convergence 
rate to a stationary point over arbitrary convex sets. The $O\left( \frac{1}{\sqrt{t}} \right)$ rate matches 
the rate of projected gradient descent when the loss function is smooth and non-convex. 
It has been shown~\cite{S16} that a point $w_t$ is a stationary 
point for the constrained optimization problem if and only if $k_t = 0$.

\begin{theorem}
\label{thm:nonconvex}
Assume that the \cameraReadyAdd{non-convex} function $f$ is smooth with parameter $L$ and the constraint set $\Omega$ is $\alpha$-strongly convex and has dimensionality $d$. 
Further, define 
the function $h$ as $h(w) = f(w) + \theta \xi^T w$ where 
\cameraReadyAdd{$\theta \in \left(0, \frac{\epsilon}{4D} \right] $},
$\xi \in \reals^d$, $w \in \Omega$, 
\cameraReadyAdd{$D$ is the diameter of $\Omega$,}
and $\xi$ is uniform on the unit sphere. Let 
$\ell_1 = f(w_1) - f(w^\ast)$ and $C' = \frac{ \alpha \delta \sqrt{\pi}}{8 L \sqrt{2 d}}$. Then
\fix{applying} FW with line search \fix{to $h$ yields the following guarantee for the FW gap of $f$} with probability $1 - \delta$,
\[
    \underset{ 1 \leq s \leq t  }{\min}  k_s \leq \frac{ \ell_1  }{ t \min \{ \frac{1}{2} ,  C' \} } = O\left( \frac{1}{t} \right)
\]
 \end{theorem}

We would further discuss the result stated in the theorem. In non-convex optimization literature, 
Nesterov and Polyak \cite{N06} show that cubic regularization of Newton's method can find a 
stationary point in $O( \epsilon^{-3/2})$ iterations and evaluations of the Hessian.
First order methods, such as gradient descent, 
typically require $O( \epsilon^{-2})$ iterations \cite{GDHS17} to converge to a stationary point.
Recent progress on first order methods, however, assumes some mild conditions and show  
that an improved rate of $O( \epsilon^{-7/4})$ is possible~\cite{GDHS17,AABHM17}.
Here, we show that when the constraint set is strongly convex, FW with line search only needs 
$O( \epsilon^{-1})$ iterations to arrive within an $\epsilon$-neighborhood of a stationary point.  
It is important to note, although the $O( \epsilon^{-1} )$ convergence rate holds probabilistically, 
	it is   quite fast compared to the known rates in the non-convex optimization literature.

\begin{figure*}[!htbp]
\centering
\begin{subfigure}[t]{0.32\linewidth}
\centering
    \includegraphics[width=1\linewidth, keepaspectratio]{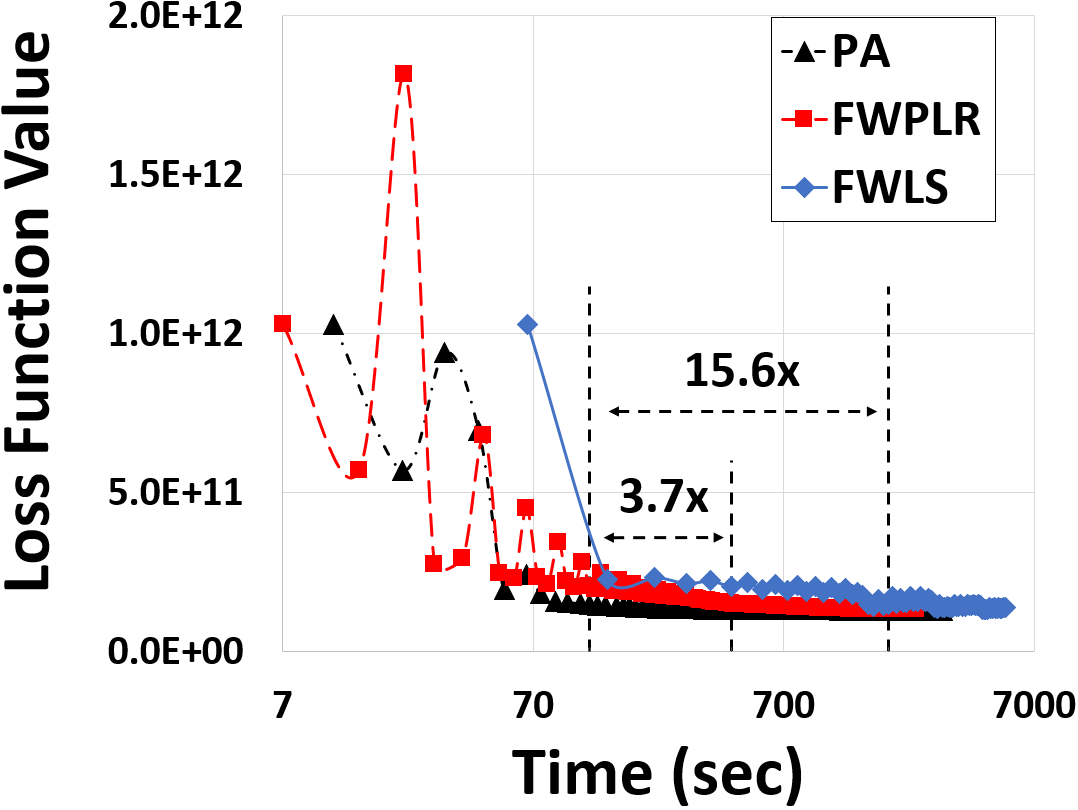}
    \centering
    \caption{PA vs. standard FW variants.}
    \label{fig:L2Time}
\end{subfigure}
\hspace{0.1cm}
\begin{subfigure}[t]{0.29\linewidth}
    \centering
    \includegraphics[width=1\linewidth, keepaspectratio]{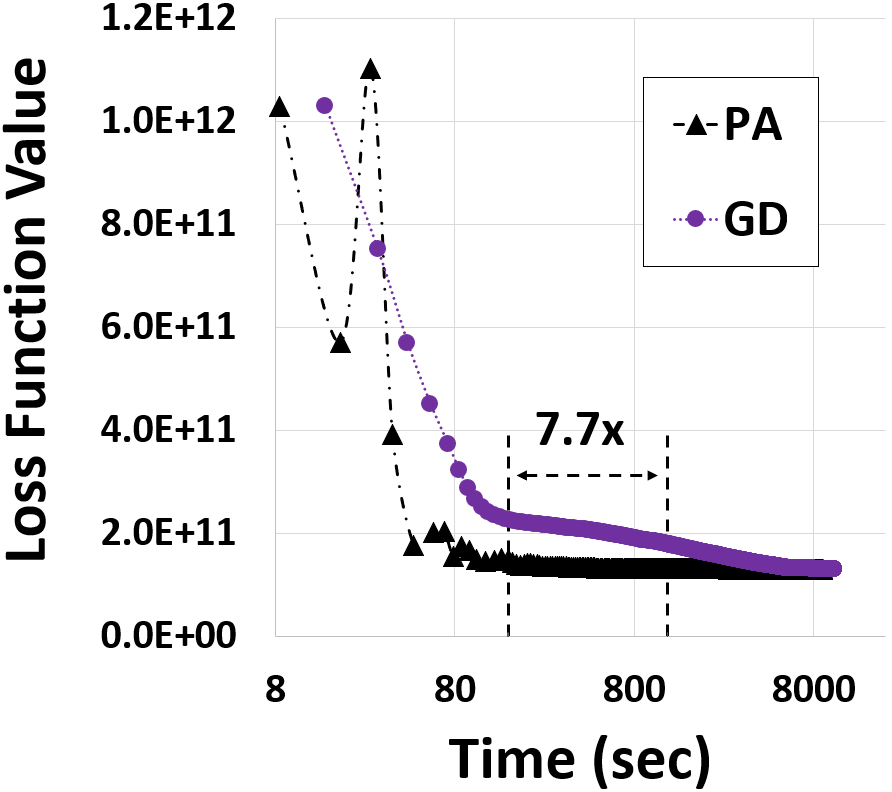}
    \caption{PA vs. gradient descent.}
    \label{fig:DetSPAVsSGD}
\end{subfigure}
\hspace{0.1cm}
\begin{subfigure}[t]{0.30\linewidth}
    \centering
    \includegraphics[width=1\linewidth, keepaspectratio]{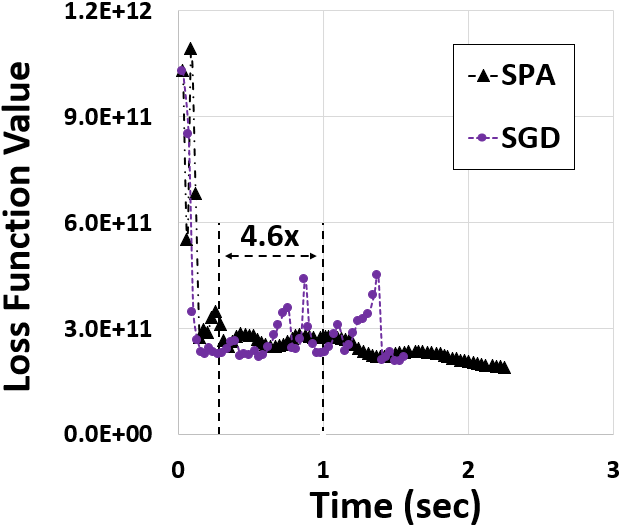}
    \caption{Stochastic PA vs. stochastic GD.}
    \label{fig:SGDGtSPA}
\end{subfigure}
\caption{PA versus (a) other FW variants, (b) gradient descent, and (c) stochastic gradient descent.}
\inv
\label{fig:times}
\end{figure*}
\section{Experiments}
\label{sec:expr}
\sinv

We have conducted extensive experiments on different combinations of loss functions, constraint sets, and real-life datasets
(Table~\ref{tbl:obj_func}). Here, we only report two main sets of experiments: the empirical validation of our theoretical results in terms of 
convergence rates (Section~\ref{sec:expr:theory}) and the comparison of various optimizations in terms of actual run times
  (Section~\ref{sec:expr:times}). We refer the interested reader to Appendix B for additional 
  experiments.
  
    For classification and regression, we used the logistic  and quadratic loss functions. 
For matrix completion, we used the \textbf{observed quadratic loss} \cite{observedMtxLoss}, 
defined as $f\left(X\right) = \sum_{(i, j) \in P(M)} (X_{i,j} - M_{i,j})^2$
 where $X$ is the estimated matrix, $M$ is the 
observed matrix, and $P(M)$$=$$\{ (i, j): M_{i, j} \textrm{ is observed}\}$. 
As a non-convex, but strictly-locally-quasi-convex loss, we also used \textbf{squared sigmoid loss} $\varphi(z) = (1 + \exp(-z))^{-1}$~\cite{HLS15} for classification.
For robust regression, we used the \textbf{bi-weight loss} \cite{biWeightLoss}, as a non-convex (but smooth) loss
$\psi(f(x_i), y_i) = \dfrac{(f(x_i) -y_i)^2}{1+(f(x_i) - y_i)^2}$.

For regression, we used the YearPredictionMSD dataset 
	(500K observations, 90 features) \cite{UCIMLRepo}.
For classification, we used the Adult dataset  (49K observations, 14 features)  \cite{UCIMLRepo}.
For matrix completion, we used the MovieLens dataset (1M movie ratings from 6,040 users on 3,900 movies) \cite{movielensDataset}.

\subsection{Empirical Validation of Convergence Rates}
\label{sec:expr:theory}

We ran several experiments to
     empirically validate our convergence results.
 In particular,
we studied the performance of Primal Averaging (\textbf{PA}) and standard FW With Line Search (\textbf{FWLS}) with
  both $l_2$ and Schatten-$2$ norm balls as our strongly convex constraint sets. 
 
 Theorem~\ref{thm:convex} guarantees a
	  convergence rate of 
	$O( \frac{1}{t^2} )$ 
    for PA
        when the constraint set is strongly convex and the loss function is convex. 
We experimented with both $l_2$  (logistic classifier) and Schatten-$2$ norm (matrix completion) balls,
  measuring the loss value at each iteration. As shown in 
 Figure~\ref{fig:MtxCmpPAConvRate}, 
    a slope of $-2.41$  confirms  Theorem~\ref{thm:convex}'s guarantee, which 
		predicts a slope of at least $-2$.

    Theorem~\ref{thm:quasi} shows that FWLS converges to
    the global minimum at the rate of $O\left( \min\left(\frac{1}{t^{1 / 3}}, \frac{1}{t^{1 / 2}}\right) \right)$
    when the constraint set is strongly convex and the loss function is 
    strictly-locally-quasi-convex.  We investigated this result with
    the squared sigmoid loss and an $l_2$ norm constraint.  Figure~\ref{fig:L2PAConvRate}
    exhibits our results, showing a slope of $-2.12$, a finding better than the
    worst-case bounds given by Theorem~\ref{thm:quasi}, i.e., a slope of $-0.5$ 
    (see appendix for a detailed discussion).

From Theorem~\ref{thm:nonconvex}, we expect  
FWLS 
 to  converge  to a stationary
    point of a (smooth) non-convex function at a rate of $O( \frac{1}{t})$
 		when constrained to a strongly convex set.  
Using the bi-weight loss and an $l_2$ norm constraint, 
    we measured the loss  value at each iteration. 
As shown in	Figure~\ref{fig:FWLSNonConvex}, the results confirmed our theoretical results, showing an even steeper slope ($-1.46$ instead of $-1$, since Theorem~\ref{thm:nonconvex} only provides a worst-case upper bound).

\subsection{Comparison of Different Optimization Algorithms}
\label{sec:expr:times}

To compare the  actual performance of various optimization algorithms,
we measure the run times, instead of the number of iterations to convergence, 
in order to account for the time spent in each iteration. 
In Figure~\ref{fig:times}, dotted vertical lines mark  the convergence points of various algorithms.

First, we compared all three variants of FW: PA,
	  	standard FW With Predefined Learning Rate (\textbf{FWPLR})
	 	defined in Algorithm~\ref{alg:fw} with option A,
        and  standard FW With Line Search (\textbf{FWLS})
		defined in   Algorithm~\ref{alg:fw} with option B.
     All methods were tested on a regression task (quadratic loss) with an
    $\ell_2$ norm ball constraint.
 
As shown in Figure~\ref{fig:L2Time},  PA converged 
 $3.7\times$ and $15.6\times$  faster than FWPLR and FWLS, respectively.  
This considerable speedup has significant ramifications in practice.
Traditionally, PA has been shied  away from, due to its slower iterations, 
	while its convergence rate was believed to be the same as the more efficient variants~\cite{L13}.
However, as   proven in Section~\ref{sec:convex}, 
	PA does   converge in fewer iterations.

We also compared the run time of PA versus projected gradient descent (regression task with a quadratic loss).
We compared their deterministic versions in Figure~\ref{fig:DetSPAVsSGD},	
    where PA converged significantly faster ($7.7\times$), as expected.  
For a fair comparison of their stochastic versions, 
	\textbf{Stochastic Primal Averaging (SPA)} and \textbf{Stochastic Gradient Descent (SGD)}, 
  we considered two cases:
	an  $l_2$ constraint (which has an efficient projection) and 
		$l_{1.1}$ constraint (which has a costly projection).
As expected, \fix{for} an efficient projection,
         	   SGD converged $4.6\times$  faster than SPA (Figure~\ref{fig:SGDGtSPA}),
	   and when the projection was costly, 
	   	SPA converged $25.1\times$ faster (see Appendix B for detailed plots).

\section{Conclusion}
\label{sec:conclusion} 
    In this paper, we revisited an important class of optimization 
    	techniques, FW methods, 		   
		and offered new insight into their  convergence properties
  			for strongly convex constraint sets, which are quite common in machine learning. 
Specifically, we discovered that, for convex functions, 
    a non-conventional variant of FW (i.e., Primal Averaging)
    	converges significantly faster 
        than the commonly used variants of FW \fix{with high probability}.
We also showed that PA's $O( \frac{1}{t^2} )$ convergence rate  more than compensates for its slightly more expensive computational cost at each iteration.
 We also proved that for strictly-locally-quasi-convex functions, FW can
    converge to within an $\epsilon$-neighborhood of the global minimum in
    $O\left(max( \frac{1}{\epsilon^2}, \frac{1}{\epsilon^3})\right)$ iterations.  
 Even for non-convex functions,  we proved 
 	that  FW's convergence rate is better than
        the previously known results in the literature \fix{with high probability}. 
These new convergence rates have significant 
	ramifications for practitioners, due 
		to the widespread applications of
            strongly convex norm constraints  in classification, regression, matrix completion, and collaborative filtering.
Finally, we conducted extensive experiments on real-world datasets
	to validate our  theoretical results
        and investigate \fix{our improvement over existing methods}.\delete{the extent of projection overhead on gradient descent approaches.}
In summary, we showed that PA reduces optimization
time by $2.8$--$15.6\times$ compared to standard FW variants,
and by \fix{$7.7$}--$25.1\times$ compared to projected gradient descent.  
Our plan is to integrate PA in machine learning libraries libraries, including our BlinkML project~\cite{blinkml-tr}.
 
\section{Acknowledgments}
\label{sec:acknowledgments} 
This work is in part supported by the National Science Foundation (grants 1629397 and 1553169).

\clearpage\newpage

\bibliographystyle{aaai}
\bibliography{fw,mozafari}  

\clearpage\newpage

\appendix

\newpage


\section{Proofs}

\subsection{Proof of Theorem~\ref{thm:convex}}
\fix{
    We begin by providing two lemmas which aid in our proof of 
    Theorem~\ref{thm:convex}.  In particular, Lemma~\ref{lm:lip}
    allows us to upper-bound the distance between two outputs of
    the linear oracle by a scaled distance of the oracle's inputs.
    Lemma~\ref{lemma:probBounds} shows that if running PA on an
    $L$-smooth function $f$ allows $f$ to converge as
    \[
        f(w_t) - f(w^*) = O\left(\frac{L}{\alpha^2 g^2 t^2} \right) 
    \]
    then running PA on a perturbed function $h$ allows $f$ to converge as
    \[
        f(w_t) - f^* = O\left(\frac{L d}{\alpha^2 \delta^2 t^2} \right) 
    \]
    with probability $1 - \delta$.  Here, 
    $g$ is the smallest value of the norm of averaged gradients and
    $f^* = \underset{w \in \Omega}{\min} f(w)$.
}

\fix{
    Given this, our proof of Theorem~\ref{thm:convex} proceeds by first
    showing that running PA on an $L$-smooth function $f$ over an 
    $\alpha$-strongly convex constraint set $\Omega$ causes $f$ to converge
    as
    \[
        f(w_t) - f(w^*) = O\left(\frac{L}{\alpha^2 g^2 t^2} \right) 
    \]
    We then apply Lemma~\ref{lemma:probBounds}, thereby showing that
    running PA on a perturbed function $h$ allows $f$ to converge as
    \[
        f(w_t) - f^* = O\left(\frac{L d}{\alpha^2 \delta^2 t^2} \right) 
    \]
    with probability $1 - \delta$.
}

\fix{
    We now state the Lemmas and provide their proofs.
}
\begin{lemma}  \label{lm:lip}
Denote  
$$x_p = \argmax_{x \in \Omega} \langle p, x\rangle $$ and $$x_q = \argmax_{x \in \Omega} \langle q, x\rangle $$ where $p,q \in \reals^d$ are any non-zero vectors.  
If a compact set $\Omega$ is an $\alpha$-strongly convex set,
then 
\begin{equation} \label{eq:lip}
    \| x_p - x_q \| \leq \frac{\norm{p - q}}{\alpha ( \| p \| + \|q \| )}
\end{equation}
\end{lemma}

\begin{proof}
    It is shown in Proposition A.1 of~\cite{HLGS16} that an $\alpha$-strongly convex set
    can be expressed as the intersection of infinitely many Euclidean balls.  Denote
    the $d$ dimensional unit sphere as
    ${\mathcal{U} = \left\{ u \in \reals^d : \norm{u}^2 = 1 \right\}}$.  
    Furthermore, write $x_u = \underset{x \in \Omega}{\argmax} \left\langle x, u \right\rangle$
    for some $u \in \mathcal{U}$.  Then the $\alpha$-strongly convex set $\Omega$ can be written
    \[ \Omega = \underset{u \in \mathcal{U}}{\cap} \mathbb{B}_{\frac{1}{\alpha}} \left( x_u - \frac{u}{\alpha} \right) \]
    Now, let 
    $${x_p = \argmax_{x \in \Omega} \langle \frac{p}{\norm{p}}, x\rangle }$$ and $${x_q = \argmax_{x \in \Omega} \langle \frac{q}{\norm{q}}, x \rangle}$$
    Based on the above interpretation of strongly convex sets, we see that
    $$x_q \in \mathbb{B}_{ \frac{1}{\alpha}  } ( x_p - \frac{p}{\alpha \norm{p}})$$ and $$x_p \in \mathbb{B}_{\frac{1}{\alpha}  } ( x_q - \frac{q}{\alpha \norm{q}} )$$
    Therefore, 
    \[
        \| x_q - x_p -  \frac{p}{\alpha \norm{p}} \|^2 \leq \frac{1}{\alpha^2}
    \]
    which leads to
    \begin{equation}
        \label{eqn:ineqSumP1}
        \| x_p - x_q \|^2 \leq \frac{2}{\alpha} \langle x_p - x_q,  \frac{p}{\| p \|} \rangle
    \end{equation}
    and
    \[
        \| x_p - x_q - \frac{q}{\alpha \|q\|} \|^2 \leq \frac{1}{\alpha^2} 
    \]
    which results in
    \begin{equation}
        \label{eqn:ineqSumP2}
        \| x_p - x_q \|^2 \leq \frac{2}{\alpha} \langle x_q - x_p,  \frac{q}{\| q \|} \rangle
    \end{equation}
    Summing (\ref{eqn:ineqSumP1}) and (\ref{eqn:ineqSumP2}) then applying the Cauchy-Schwarz inequality completes the proof.
\end{proof}

\fix{
    We note that Lemma~\ref{lm:lip} also implies that the sub-gradient 
    of the linear oracle's output is Lipschitz continuous.  We now present
    Lemma~\ref{lemma:probBounds}.
}

\begin{lemma} \label{lemma:probBounds}
 Given $f: \reals^d \rightarrow \reals, w, \xi \in \reals^d$, and $\theta \in \reals$, denote $h(w) = f(w) + \theta \xi^T w$. Let $\xi$ be 
uniform on the unit sphere and $\theta > 0$. Then with probability $1 - \delta$,

\[
    \frac{\delta \sqrt{\pi}}{\sqrt{2d}} \leq \norm{\nabla h(w)}
\]

Further, if $h$ is smooth with parameter $L$ and applying PA 
    to $h$ over an $\alpha$-strongly convex constraint set $\Omega$
    \fix{
        allows $h$ to converge as
        \[
            h(w_t) - h^* = O\left(\frac{L}{\alpha^2 g^2 t^2} \right) 
        \]
        then it also holds that
    }
\[
    f(w_t) - f^* = O\left(\frac{L d}{\alpha^2 \delta^2 t^2} \right) 
\]

with probability $1 - \delta$.
\end{lemma}

\begin{proof}
Note that $\norm{\nabla h(w)} \geq \left|\nabla h(w)[1]\right|$ where $y[i]$ is the $i$th component 
of $y \in \reals^d$. Since $\nabla h(w)[1] = \nabla f(w)[1] + \theta \xi[1]$,

\[
	Pr\left(\nabla h(w)[1] \leq \frac{\theta \delta \sqrt{\pi}}{\sqrt{2d}}\right) \leq Pr\left(\xi[1] \leq \frac{\delta \sqrt{\pi}}{\sqrt{2d}}\right) \leq \delta
\]

as $\xi[1] \sim Beta(\frac{1}{2}, \frac{d-1}{2})$. Thus, we can lower bound the gradient norm by $\frac{\delta \sqrt{\pi}}{\sqrt{2d}}$ with 
probability $1 - \delta$. 
    
\fix{
    Now, denote
    \begin{equation*}
        \begin{aligned}
            h^* &= \underset{w \in \Omega}{\min} h(w) \\
            w^* &= \underset{w \in \Omega}{\argmin} f(w)
        \end{aligned}
    \end{equation*}
    Recall our assumption that $h$ is smooth 
    with parameter $L$ and that applying PA to $h$ causes $h$ to converge as
    \[
        h(w_t) - h^* = O\left(\frac{L}{\alpha^2 g^2 t^2} \right)
    \]
    Note that $g \geq \frac{\delta \sqrt{\pi}}{\sqrt{2d}} \geq 0$, so
    \[
        h(w_t) - h^* = O(\frac{L d}{\alpha^2 \delta^2 t^2})
    \]
    Further, observe that by the construction of $h$ we have 
    \begin{equation} \label{eqn:hConstr}
        \left| f(w) - h(w)\right| \leq \theta \norm{w} 
    \end{equation}
    for any $w$.  Now, whenever 
    $h(w_t) - h^* = O(\frac{L d}{\alpha^2 \delta^2 t^2}) = \epsilon$ we see that
}

\begin{equation*}
	\begin{aligned}
		f(w_t) &\leq f(w_t) + \theta \norm{w_t} \\
		&\leq h^* + \epsilon + \theta \norm{w_t} \\
		&\leq h(w^*) + \epsilon + \theta \norm{w_t} \\
		&\leq f(w^*) + \epsilon + \theta (\norm{w_t} + \norm{w^*}) \\
		&\leq f(w^*) + \epsilon + 2 \theta D \\
		&\leq f(w^*) + 1.5 \epsilon
	\end{aligned}
\end{equation*}

The first and third lines follow from (\ref{eqn:hConstr}), the second from $h^*$ being the minimum value of $h$, the fourth 
from $D$ being the diameter of the constraint set, and fifth from the choice of $\theta = \frac{\epsilon}{4 D}$. 
    Thus we obtain the convergence rate of $O\left(\frac{L d}{\alpha^2 \delta^2 t^2}\right)$ \fix{for $f$} with probability $1 - \delta$.
\end{proof}

\fix{
    We now proceed with our proof of Theorem~\ref{thm:convex}.
}

\begin{proof}

According to Theorem 8 in \cite{L13}, we already have
\begin{equation} \label{eq:lan}
f( w_t)  - f( w^\ast) \leq  
\frac{2 L}{ t (t+1) } \sum_{\tau = 1}^t \norm{ v_{\tau} - v_{\tau-1}}^2
\end{equation}
Fix a $t$. Denote 
\begin{equation*}
    \begin{aligned}
        p_t = \frac{1}{\Theta_t} \sum_{i=1}^t \theta_i \nabla f( z_{i-1})
    \end{aligned}
\end{equation*}
and 
\begin{equation*}
    \begin{aligned}
        p_{t-1}= \frac{1}{\Theta_{t-1}} \sum_{i=1}^{t-1} \theta_i \nabla f(z_{i-1})
    \end{aligned}
\end{equation*}
By using Lemma~\ref{lm:lip}, we have
\begin{equation} \label{eq:r}
    \begin{aligned}
        \| v_{t} - v_{t-1} \| \leq \frac{\norm{p_t - p_{t - 1}}}{\alpha (\| p_t \| + \| p_{t-1}\| ) }
    \end{aligned}
\end{equation}
Based on the update rule 
\begin{equation}
\begin{aligned}
p_t & = \Theta_t^{-1} ( p_{t-1} \Theta_{t-1} + \theta_t \nabla f(z_{t-1}) ) \\ & = \frac{ \Theta_t - \theta_t }{\Theta_t} p_{t-1} + \frac{\theta_t}{\Theta_t} \nabla f(z_{t-1})
\end{aligned}
\end{equation}
So
$$p_t - p_{t-1} = \gamma_t ( \nabla f(z_{t-1}) - p_{t-1} )$$
given that $\gamma_t = \frac{\theta_t}{\Theta_t}$.
By substituting the result back into (\ref{eq:r}) and noting that $\gamma_t = O\left( \frac{1}{t} \right)$, we find that
\begin{equation} \label{eq:t1}
    \begin{aligned}
        \| v_{t} - v_{t-1} \| &\leq \frac{\gamma_t( \| p_{t-1} \| + \| \nabla f(z_{t-1}) \| )}{\alpha (\| p_t
        \| + \| p_{t-1}\| ) } \\
        &= O( \frac{1}{ \alpha g t} ), \forall t
    \end{aligned}
\end{equation}
    \delete{where $g$ is the smallest value of the norm of averaged gradients.}  By combining (\ref{eq:lan}) and (\ref{eq:t1}), we get 

\begin{equation}
\begin{aligned}
f( w_t)  - f( w^\ast) & \leq  
\frac{2 L}{ t (t+1) } \overset{t}{\underset{ \tau=1}{ \Sigma} } \| v_{\tau} - v_{\tau-1} \|^2 \\& \leq \frac{2 L}{ t (t+1) } \Sigma_{i=1}^t O( \frac{1}{ \alpha^2 g^2 i^2} ) \\
    &= O( \frac{L}{\alpha^2 g^2 t^2}) 
\end{aligned}
\end{equation}
where we used the fact that $\sum_{i = 1}^{\infty} \frac{1}{i^2} = \frac{\pi}{6}$ in the final equality.

Finally, applying the result of Lemma~\ref{lemma:probBounds} yields with probability $1 - \delta$ the convergence rate
 $f(w_t) - f(w^*) = O\left(\frac{L d}{\alpha^2 \delta^2 t^2}\right)$ as claimed.
\end{proof}

\subsection{Proof of Theorem~\ref{thm:spa}}
\begin{proof}
We note that
\begin{equation} \label{e01}
 \begin{aligned}
     f( \mathbf{w}_{t} ) &\leq f( \mathbf{z}_{t-1} ) + \langle \nabla f( \mathbf{z}_{t-1} ) , \mathbf{w}_{t}   - \mathbf{z}_{t-1} \rangle 
 \\ &\quad+ \frac{L}{2} \|  \mathbf{w}_{t} - \mathbf{z}_{t-1} \|^2
 \\ &\overset{(0)}{=} ( 1 - \gamma_t)  [ f( \mathbf{z}_{t-1} ) + \langle \nabla f( \mathbf{z}_{t-1} ) ,  \mathbf{w}_{t-1}  - \mathbf{z}_{t-1} \rangle]
 \\&\quad+ \gamma_t [ f( \mathbf{z}_{t-1} ) + \langle \nabla f( \mathbf{z}_{t-1} )   , 
  \mathbf{v}_{t}  - \mathbf{z}_{t-1}  \rangle ]  
 \\&\quad+ \frac{L \gamma_t^2}{2} \|  \mathbf{v}_{t} - \mathbf{v}_{t-1} \|^2 
 \\ &\overset{(1)}{ \leq} ( 1 - \gamma_t) f( \mathbf{w}_{t-1} ) +  \gamma_t [ f( \mathbf{z}_{t-1} ) 
 \\ &\quad+ \langle \nabla f( \mathbf{z}_{t-1} ) , \mathbf{v}_{t}   - \mathbf{z}_{t-1} \rangle] + \frac{L \gamma_t^2}{2} \|  \mathbf{v}_{t} - \mathbf{v}_{t-1} \|^2
 \\ &= ( 1 - \gamma_t) f( \mathbf{w}_{t-1} ) +  \gamma_t [ f( \mathbf{z}_{t-1} ) 
 \\ &\quad+ \langle \tilde{\nabla} f( \mathbf{z}_{t-1} ) , \mathbf{v}_{t}   - \mathbf{z}_{t-1} \rangle] + \frac{L \gamma_t^2}{2} \|  \mathbf{v}_{t} - \mathbf{v}_{t-1} \|^2 
 \\&\quad+ \gamma_t \langle  \nabla f(\mathbf{z}_{t-1} ) - \tilde{\nabla}f( \mathbf{z}_{t-1} )  , \mathbf{v}_{t}   - \mathbf{z}_{t-1} \rangle
 \\ &\overset{(2)}{ \leq} ( 1 - \gamma_t) f( \mathbf{w}_{t-1} ) +  \gamma_t [ f( \mathbf{z}_{t-1} ) 
 \\&\quad+ \langle \tilde{\nabla} f( \mathbf{z}_{t-1} ) , \mathbf{w}^\ast   - \mathbf{z}_{t-1} \rangle] + \frac{L \gamma_t^2}{2} \|  \mathbf{v}_{t} - \mathbf{v}_{t-1} \|^2 
 \\&\quad+ \gamma_t \langle  \nabla f(\mathbf{z}_{t-1} ) - \tilde{\nabla}f( \mathbf{z}_{t-1} )  , \mathbf{v}_{t}   - \mathbf{z}_{t-1} \rangle
 \\ &= ( 1 - \gamma_t) f( \mathbf{w}_{t-1} ) +  \gamma_t [ f( \mathbf{z}_{t-1} ) 
 \\ &\quad+ \langle \nabla f( \mathbf{z}_{t-1} ) , \mathbf{w}^\ast   - \mathbf{z}_{t-1} \rangle] + \frac{L \gamma_t^2}{2} \|  \mathbf{v}_{t} - \mathbf{v}_{t-1} \|^2 
 \\ &\quad+ \gamma_t \langle \nabla f(\mathbf{z}_{t-1} ) - \tilde{\nabla}f( \mathbf{z}_{t-1} )  , \mathbf{v}_{t}   - \mathbf{w}^\ast \rangle 
 \\ &\overset{(3)}{ \leq} ( 1 - \gamma_t) f( \mathbf{w}_{t-1} ) +  \gamma_t f( \mathbf{w}^\ast ) 
 \\ &\quad+ \gamma_t \langle \nabla f(\mathbf{z}_{t-1} ) - \tilde{\nabla}f( \mathbf{z}_{t-1} )  , \mathbf{v}_{t}   - \mathbf{w}^\ast \rangle 
 \\ &\quad+ \frac{L \gamma_t^2}{2} \|  \mathbf{v}_{t} - \mathbf{v}_{t-1} \|^2 
 \end{aligned}
\end{equation}
$(0)$ follows from the fact that, as 
$\mathbf{w}_{t} = (1 - \gamma_t) \mathbf{ w }_{t-1} + \gamma_t \mathbf{v}_t$ and 
$\mathbf{z}_{t-1} = (1 - \gamma_t) \mathbf{w}_{t-1} + \gamma_t \mathbf{v}_{t-1}$,
\begin{equation} 
\mathbf{w}_{t} - \mathbf{z}_{t-1} = \gamma_t ( \mathbf{v}_t - \mathbf{v}_{t-1} )
\end{equation}
Furthermore, $(1)$ is implied by the convexity of $f$, $(2)$ follows from the application of the
linear oracle, and $(3)$ follows again from the convexity of $f$. Moreover, by taking the 
expectation over the randomness, we find that
\begin{equation} \label{e02}
 \begin{aligned}
     E[ f( \mathbf{w}_{t} ) ] &\leq 
 ( 1 - \gamma_t) f( \mathbf{w}_{t-1} ) +  \gamma_t f( \mathbf{w}^\ast )  
 \\&\quad+ \frac{L \gamma_t^2}{2} \|  \mathbf{v}_{t} - \mathbf{v}_{t-1} \|^2 + \gamma_t \frac{ \sigma D }{ \sqrt{n_t}} 
\end{aligned}
\end{equation}
since 
\begin{equation}
    \begin{aligned}
        E[ \| &\nabla f(\mathbf{z}_{t-1} ) - \tilde{\nabla}f( \mathbf{z}_{t-1} ) \| ] 
        \\&\leq \sqrt{ E[ \| \nabla f(\mathbf{z}_{t-1} ) - \tilde{\nabla}f( \mathbf{z}_{t-1} ) \| ]^2 } 
        \\&\leq \sigma / \sqrt{n_t}
    \end{aligned}
\end{equation}

To maintain an $O(\frac{1}{t^2})$ convergence rate, $\frac{ \sigma D \gamma_t }{ \sqrt{n_t}}$ must decay as  
${\frac{L \gamma_t^2}{2} \|  \mathbf{v}_{t} - \mathbf{v}_{t-1} \|^2 }$.
Recall that the latter term is $O\left(\frac{1}{t^4}\right)$.
This implies that $n_t$ must be $O(t^6)$ so that $\gamma_t \frac{ \sigma D }{ \sqrt{n_t}}$
can decay as $O\left(\frac{1}{t^4}\right)$.
However, if $n_t = O(t^4)$, stochastic Primal Averaging yields the slightly
worse $O\left(\frac{\log t}{t^2}\right)$ convergence rate.

Finally, we can remove the reliance on the 
$\underset{1 \leq s \leq t}{\min} \norm{\nabla f( z_s )} + \norm{\nabla f( z_{s - 1} )}$
term in the convergence rate by repeating the analysis given in Lemma~\ref{lemma:probBounds}.
\fix{
    As this analysis is very similar to the previously provided analysis of
    Lemma~\ref{lemma:probBounds}, it is omitted.
}
Then when $n_t = O(t^4)$ we have with probability $1 - \delta$ 

\begin{equation*}
    E[ f(w_t) ] - f^* = O\left( \frac{ d L^2 (D^2 + \sigma) \log t }{\alpha^2 \delta^2 t^2} \right)
\end{equation*}

\end{proof}

\subsection{Proof of Theorem~\ref{thm:quasi}}
We state the following lemma from~\cite{D15}.
\begin{lemma} 
\label{lem:qu} 
Write the dual norm as $\norm{\cdot}_{\ast}$.  For iteration $t$ of FW with line search, if 
${L < \frac{\alpha \| \nabla f( w_t ) \|_{\ast} }{4}}$ set $\gamma_t = 1$; otherwise, set 
${\gamma_t = \frac{\alpha  \| \nabla f( w_t ) \|_{\ast} }{4 L}}$.
Then, under the conditions of Theorem~\ref{thm:quasi}, Algorithm~\ref{alg:fw} option (B) 
has the following guarantee:
\[
f( w_{t+1} ) \leq f( w_t) + \frac{\gamma_t }{2} \langle w^\ast - w_t, \nabla f( w_t ) \rangle
\]
\end{lemma}
We use Lemma~\ref{lem:qu} to prove Theorem~\ref{thm:quasi}.
\begin{proof}
Assume that,
$$f(w_t)-f(w^*) > \epsilon$$
Otherwise, the algorithm has reached the $\epsilon$-neighborhood of $w^*$. 
By the strictly-locally-quasi-convexity of $f$, we must have,
    $$\| \nabla f(w) \| > 0$$ and for every $x \in \mathbb{B}_{\epsilon / \kappa}(w)$ it holds that,
$$\langle \nabla f(w), x-w \rangle \leq 0$$ 
Now choose a point $y$ such that, 
$$y = w^* + \frac{\epsilon \nabla f(w_t)}{\kappa \| \nabla f(w_t) \|}$$ 
and,
$$ y \in \mathbb{B}_{\epsilon / \kappa}( w^* )$$
Then we have the following,
\begin{equation}\label{eq:5}
\begin{aligned}
&\hspace{1.33em}\relax   \langle \frac{\nabla f(w_t)}{\| \nabla f(w_t) \|}, y - w_t \rangle  \leq  0
\\& \equiv \langle \frac{\nabla f(w_t)}{\| \nabla f(w_t) \|}, \frac{\epsilon \nabla f(w_t)}{\kappa \| \nabla f(w_t) \|} + w^* - w_t \rangle \leq 0
\\& \equiv \langle \frac{\nabla f(w_t)}{\| \nabla f(w_t) \|}, w_t - w^*\rangle \geq \frac{\epsilon}{\kappa} 
\\& \equiv \langle \nabla f(w_t), w_t - w^*\rangle \geq \frac{\epsilon}{\kappa} \| \nabla f(w_t) \| 
\end{aligned}
\end{equation}

Case 1: ($L < \frac{\alpha \| \nabla f( w_t ) \|_{\ast} }{4}$. Set $\gamma_t = 1$):
\begin{equation} \label{res:1}
\begin{aligned}
f( w_{t+1} ) &\leq f(w_t) + \frac{\gamma_t }{2} \langle \nabla f(w_t), w^\ast - w_t \rangle
\\ &  \overset{(\ref{eq:5})}{\leq} f(w_t) - \frac{\gamma_t \epsilon}{2 \kappa} \| \nabla f(w_t) \| 
\\ &= f(w_t) - \frac{\epsilon}{2 \kappa} \| \nabla f(w_t) \| 
\end{aligned}
\end{equation}

Case 2: ($L \geq \frac{\alpha \| \nabla f( w_t ) \|_{\ast} }{4}$. Set $\gamma_t = \frac{\alpha  \| \nabla f( w_t ) \|_{\ast} }{4 L}$):
\begin{equation} \label{res:2}
\begin{aligned}
f( w_{t+1} ) &\leq f(w_t) + \frac{\gamma_t }{2} \langle \nabla f(w_t), w^\ast - w_t \rangle
\\ & = f(w_t) + \frac{\alpha  \| \nabla f( w_t ) \|_{\ast} }{16 L } \langle \nabla f(w_t), w^\ast - w_t \rangle
 \\ & \overset{(\ref{eq:5})}{=} f(w_t) - \frac{\alpha \epsilon \| \nabla f( w_t ) \|_{\ast} \| \nabla f(w_t) \|}{8 \kappa L }  
\\ & \leq f(w_t) - \frac{\alpha \epsilon \| \nabla f(w_t) \|^2_2 }{8 \kappa L } 
\end{aligned}
\end{equation}
By (\ref{res:1}) and (\ref{res:2}), we observe that the loss function 
monotonically decreases until it enters an $\epsilon$-neighborhood of the global minimum, 
thereby proving that FW with line search can converge 
within an $\epsilon$-neighborhood of the global minimum. 
To prove that the algorithm requires $t= O(max( \frac{1}{\epsilon^2}, \frac{1}{\epsilon^3}) )$, 
we use the additional assumption,
$$ f(w) - f(w^*) \geq \epsilon \rightarrow \| \nabla f(w) \| \geq \theta \epsilon$$ 
Now, assume that after iteration $t$ the algorithm reaches the target $\epsilon$-neighborhood.
Denote the solution vector at iteration $t$ as $w_t$.  Then, in case (1), we have,
\begin{equation}
f(w_t) \leq f(w_1) - \frac{t \epsilon^2 \theta}{2 \kappa}
\end{equation}
 or 
 \begin{equation}
    \begin{aligned}
        t &\leq \frac{ 2 \kappa (f(w_1) - f(w_t)) }{ \epsilon^2 \theta } 
        \\ &\leq \frac{ 2 \kappa (f(w_1) - f(w^*)) }{ \epsilon^2 \theta }
    \end{aligned}
\end{equation}
while for case (2), we have,
\begin{equation}
f(w_t) \leq f(w_1) - \frac{t \epsilon^3 \theta}{8 \kappa L}
\end{equation} or 
\begin{equation}
    \begin{aligned}
        t &\leq \frac{ 8 L \kappa (f(w_1) - f(w_t)) }{ \epsilon^3 \theta } 
        \\ &\leq \frac{ 8 L \kappa (f(w_1) - f(w^*)) }{ \epsilon^3 \theta }
    \end{aligned}
\end{equation}
This shows that it requires 
${t = O\left(max\left( \frac{1}{\epsilon^2}, \frac{1}{\epsilon^3}\right) \right)}$ 
iterations for the algorithm to produce an iterate that is within the 
$\epsilon-$neighborhood of the global minimum.

\end{proof}

\subsection{Proof of Theorem~\ref{thm:nonconvex}}
\begin{proof}  
Let $w_{\gamma} =   w_t + \gamma ( p_t - w_t)$ for some $\gamma \in [0, 1]$. Then, $f(w_{t+1}) \leq f(w_{\gamma})$ as $w_{t+1}$ is obtained by line search
and thus uses an optimal step size.

\begin{equation}\label{eq:1}
\begin{aligned}
    f(w_{t+1}) &\leq f(w_{\gamma})
\\&\leq f(w_t) + \gamma \langle \nabla f(w_t) , v_t - w_t \rangle 
\\&\quad+ \frac{\gamma^2 L \norm{v_t - w_t}^2}{2}  
\\& \overset{(0)}{\leq} f(w_t) + \gamma \langle c_t - w_t, \nabla f(w_t) \rangle 
\\&\quad + \frac{\gamma^2 L \norm{v_t - w_t}^2}{2}  \\
\end{aligned}
\end{equation}
    where (0) follows from $v_t = \argmin_{v \in X} \langle \nabla f(w_t) , v \rangle$.
    Let $c_t$ above be, $$c_t = \frac{u_t + v_t}{2} + \frac{ \alpha w_t \norm{w_t - v_t}^2}{ 8 }$$ 
    where $c_t \in \Omega$ by the definition of a strongly convex set.
Let us write,
$$u_t = \argmin_{ \| u \| \leq 1 } \langle u, \nabla f(w_t) \rangle = - \| \nabla f( w_t ) \|_{\ast}$$
where the last equality is obtained by the definition of the dual norm.
Then, 
\begin{equation}\label{eq:2}
 \begin{aligned}
\langle  c_t - w_t , \nabla f( w_t  ) \rangle  &\leq \frac{1}{2} \langle v_t - w_t , \nabla f( w_t  ) \rangle 
\\&\quad+ \frac{ \alpha}{ 8 } \langle \|  v_t - w_t \|^2 u_t, \nabla f( w_t  ) \rangle
\\&\leq \frac{1}{2} \langle v_t - w_t , \nabla f( w_t  ) \rangle 
\\&\quad- \frac{ \alpha}{ 8 } \|  v_t - w_t \|^2 \| \nabla f( w_t ) \|_{\ast}\\
&=   - \frac{k_t}{2}  - \frac{ \alpha}{ 8 } \|  v_t - w_t \|^2 \nabla \| f( w_t ) \|_{\ast}\\
  \end{aligned}
\end{equation}
where the last line is due to the definition of the FW gap.  
Combining (\ref{eq:1}) and (\ref{eq:2}) gives,
\[
f( w_{t+1} ) \leq f( w_t) - \frac{\gamma k_t}{2} + \frac{ \|  v_t - w_t \|^2 }{2} ( \gamma^2 L - \gamma \frac{\alpha \| \nabla f( w_t ) \|_{\ast} }{ 4 } )
\]
Case 1: $L \leq \frac{\alpha  \| \nabla f( w_t ) \|_{\ast} }{4}$, set $\gamma=1$, we get,
\[
f( w_{t+1} ) \leq f( w_t) - \frac{k_t}{2} 
\]
Case 2: $L \geq \frac{\alpha  \| \nabla f( w_t ) \|_{\ast} }{4}$, set $\gamma = \frac{\alpha  \| \nabla f( w_t ) \|_{\ast} }{4 L}$, we get,
\[
f( w_{t+1} ) \leq f( w_t) -  \frac{\alpha  k_t \| \nabla f( w_t ) \|_{\ast} }{8 L} 
\]
By recursively applying the above inequality, we get,
\[
    f( w_{t+1} ) \leq f( w_1) - \sum_{s=1}^t \min \left( \frac{k_s}{2},  \frac{\alpha k_s \| \nabla f( w_s ) \|_{\ast} }{8 L} \right)
\]
Denote,
$    \tilde{k}_t = \underset{ 1 \leq s \leq t }{ \min } k_s. $
We have,
\begin{equation} \label{eq: main}
    f( w_{t+1} ) \leq f( w_1) - \tilde{k}_t \sum_{s=1}^t \min \left( \frac{1}{2} ,  \frac{\alpha  \| \nabla f( w_s ) \|_{\ast} }{8 L} \right)
\end{equation}
Furthermore, if we assume that,
\begin{equation} 
\| \nabla f( w_s ) \|_{\ast} \geq c > 0 , \forall s
\end{equation}
then,
\[
  \tilde{k}_t \leq  \frac{ ( f( w_1) - f( w_t) )  }{ t \min \{ \frac{1}{2} ,  \frac{\alpha c }{8 L} \} } 
\]
Since $ f( w_1) - f( w_t)  \leq  f(w_1) - f(w^\ast) = \fix{\ell_1}$, we get,
\[
    \tilde{k}_t \leq  \frac{ \fix{\ell_1}  }{ t \min \{ \frac{1}{2} ,  C' \} }
\]
where $C' = \frac{\alpha c }{8 L}$.

Finally, by repeating the analysis of Lemma~\ref{lemma:probBounds} we find with probability
$1 - \delta$ that $c \leq \frac{\delta \sqrt{\pi}}{\sqrt{2d}}$ and thus 
\[
    \tilde{k}_t \leq  \frac{ \fix{\ell_1}  }{ t \min \{ \frac{1}{2} ,  C' \} }
\]
for $C' = \frac{\alpha \delta \sqrt{\pi}}{8 L \sqrt{2 d}}$.

\end{proof}


\section{Additional Experiments}
\label{app:add-expr}

In this appendix, we provide a more detailed version of our experimental results.
Our experiments aim to answer the following questions:
\begin{enumerate}
    \item In what situations do the projections become a performance bottleneck for gradient descent algorithms?  (Section~\ref{app:sec:overhead})

\item When optimizing convex functions over strongly convex sets, does Primal Averaging (PA) outperform standard FW in practice (as 
    our theory from Section~\ref{sec:convex} suggests)? If so, by how much?  (Sections~\ref{app:sec:PAConv} and~\ref{app:sec:PAPerf})

\item 
    Does PA also outperform projected gradient descent when optimizing convex functions over
    strongly convex sets?  If so, by how much?  (Section~\ref{app:sec:PAVsGD})

\item For strictly-locally-quasi-convex
	 loss functions~\cite{HLS15}, 
		does FW's  convergence rate in practice match our theoretical results from Section~\ref{sec:quasi}? 
        (Section~\ref{sec:FWSLQCExp})

\item When optimizing non-convex loss functions, 
    		how fast does FW converge in practice? Does it match our results from Section~\ref{sec:nonconvex}? 
            (Section~\ref{sec:NonCVExp})
\end{enumerate}

In summary, our empirical results show the following:
\begin{enumerate}
    \item Projections are costly and responsible for a considerable portion of the overall runtime of gradient descent,
	whenever
    the projection step has no closed-form solution (e.g., when the constraint set is an $l_{1.5}$ ball),  
    or
    when the closed-form solution itself is expensive (e.g., projecting a matrix onto a 
    nuclear norm ball, which requires computing the SVD~\cite{NuclearNormProjection}).
\item 
    In practice, the convergence rate of primal averaging matches our theoretical
    result of $O( \frac{1}{t^2} )$ for smooth, convex functions with a strongly
    convex constraint set.  Furthermore, under these conditions, primal averaging
    outperforms  FW both with and without line search by 
    $3.7$--$15.6\times$ for a regression task and 
    $2.8$--$11.7\times$ for a
    matrix completion task 
    in terms of the overall optimization time. It also outperforms projected
    gradient descent by $7.7\times$ and can outperform stochastic gradient descent by up to $25.1\times$.
    
\item 
    When optimizing strictly-locally-quasi-convex functions over strongly
    convex constraint sets, FW with line search converges
    to an $\epsilon$-neighborhood of the global minimum within 
    $O( \max(\frac{1}{\epsilon^2}, \frac{1}{\epsilon^3}) )$ iterations, as predicted
    by Theorem~\ref{thm:quasi}.

\item 
    When the loss function is non-convex but the constraint set is strongly convex,
    FW with line search converges to a stationary point at a convergence rate of $O( \frac{1}{t} )$,
    as predicted by Theorem~\ref{thm:nonconvex}.

\end{enumerate}

\subsection{Experiment Setup}
\label{app:sec:expr:setup}

\ph{Hardware and Software} Unless stated otherwise,
	all experiments were conducted on a
	Red Hat Enterprise Linux 7.1
	 server with 112 Intel(R) Xeon(R) CPU E7-4850 v3 processors and 2.20GHz cores and 1T DDR4 memory.
	 All algorithms were implemented in Matlab R2015a. 
For projections that
required solving a convex optimization problem, we used the CVX  package~\cite{CVXPackage,CVXPaper}.

\begin{table*}[!th]
\centering
\begin{tabular}{|c|c|c|c|}
\hline
\textbf{Convexity of Loss Function} & \textbf{Loss Function} & \textbf{Constraint} & \textbf{Task}  \\ \hline
\multirow{2}{*}{Convex} 
 & Logistic Loss & $l_p$ norm & Classification \\ \cline{2-4}
& Quadratic Loss & $l_p$ norm & Regression \\ \cline{2-4}
 & Observed Quadratic Loss & Schatten-$p$ norm  & Matrix Completion \\ \hline
Strictly-Locally-Quasi-Convex & Squared Sigmoid & $l_p$ norm & Classification \\ \hline
Non-Convex & Bi-Weight Loss & $l_p$ norm & Robust Regression\\ \hline
\end{tabular}
\caption{Various loss  functions and constraint sets used in our experiments.}
\label{app:tbl:obj_func}
\end{table*}

\ph{Loss Functions} In our experiments, we used a variety of popular loss functions 
	 to cover various types of convexity and different types of machine learning tasks used in practice. 
	 These functions, summarized in Table~\ref{app:tbl:obj_func}, are as follows.
\begin{itemize}

\item \textbf{Logistic Loss.}
Logistic regression uses a convex 
loss 
 function, which is also
	commonly used in classification tasks~\cite{buja2005loss} 
	and is defined as: 
\begin{equation} \label{eqn:logistic_loss}
\ell(f(x_i), y_i) = \log(1 + e^{-y_if(x_i)})
\end{equation}
where $f$ is a hypothesis function for the learning task and $y_i$ is the target value corresponding to $x_i$.
 Logistic loss is often used with an $l_p$ norm
  constraint to avoid  overfitting~\cite{huang2011gradient}. 
 The optimization problem is thus stated as follows:
\begin{equation} \label{eqn:classification}
\begin{aligned}{}
& \underset{w \in \mathbb{R}^{d}, b \in \mathbb{R}}{\min}  \text{ }  \sum_{i=1}^N \ell\left(w^T x_i + b, y_i\right) \\
& \text{ s.t. } \| w \|_p \leq r.
\end{aligned}
\end{equation}
where $w$ is the coefficient vector, $b$ is
the linear offset, 
 $N$ is the number of data points, and $r$ is the radius of the $l_p$ norm ball.

\item \textbf{Quadratic Loss.} 
The quadratic loss is a convex loss function and is commonly used in regression tasks (a.k.a. least squares loss)~\cite{neter1996applied}:
\begin{equation} \label{eqn:quadratic_loss}
\psi(f(x_i), y_i) = (f(x_i)-y_i)^2
\end{equation}
Similar to logistic regression, a typical choice of constraint here is the $l_p$ norm.
The optimization is stated as follows:
\begin{equation} \label{eqn:regression}
\begin{aligned}{}
& \underset{w \in \mathbb{R}^{d}, b \in \mathbb{R}}{\min}  \text{ }  \sum_i \psi\left(w^Tx_i + b, y_i\right) \\
& \text{ s.t. } \| w \|_p \leq r.
\end{aligned}
\end{equation}

\item \textbf{Observed Quadratic Loss.}  
This loss function is also convex, but is typically used in matrix completion tasks~\cite{observedMtxLoss}, 
and is defined as:
\begin{equation} \label{eqn:observed_quadratic_loss}
\|X - M \|_{OB}^2 = \sum_{(i, j) \in P(M)} (X_{i,j} - M_{i,j})^2,
\end{equation}
where $X, M \in \mathbb{R}^{m \times n}$, $X$ is the estimated matrix, $M$ is the 
observed matrix, and $P(M) = \{ (i, j) : M_{i, j} \textrm{ is observed}\}$. 
In matrix completion, the loss function is often constrained within a Schatten-$p$ norm ball~\cite{schattenUsedOft1,schattenUsedOft2,schattenUsedOft3}, 
	 which is a convex constraint set.
Here, the optimization problem is stated as follows:
\begin{equation} \label{eqn:mc}
\begin{aligned}{}
& \underset{X \in \mathbb{R}^{m\times n}}{\min}  \text{ } \| {X - M} \|_{OB}^2 \\
& \text{ s.t. } \| X \|_{\mathbb{S}_p} \leq r
\end{aligned} 
\end{equation}
where $\| \cdot \|_{\mathbb{S}_p}$ is the Schatten-$p$ norm.
 
\item \textbf{Squared Sigmoid Loss.} 
    This function is non-convex\delete{~\cite{HLS15}},
    but it is strictly-locally-quasi-convex (see 
    \delete{Section~\ref{sec:quasi}}\fix{Section 3.1.1 of~\cite{HLS15}}),
	and is defined as:
    \begin{equation} \label{eqn:sigmoid}
    \varphi(z) = (1 + \exp(-z))^{-1}
    \end{equation}
where $z \in \mathbb{R}^n$.  We can state the optimization problem 
    as follows:
    \begin{equation} 
    \begin{aligned}{}
    & \underset{w \in \mathbb{R}^{d}, b \in \mathbb{R}}{\min}  \text{ }  \frac{1}{n} \sum_i^n (y_i - \varphi\left(w^Tx_i + b)\right)^2 \\
    & \text{ s.t. } \| w \|_p \leq r.
    \end{aligned}
    \end{equation}
    where $n$ is the number of data points.

\item \textbf{Bi-Weight Loss.} 
This loss function is non-convex, and is defined as follows: 
\begin{equation} \label{eqn:robust_loss}
\phi(f(x_i), y_i) = \dfrac{(f(x_i) -y_i)^2}{1+(f(x_i) - y_i)^2}
\end{equation}
 The bi-weight loss is typically used for robust regression tasks~\cite{biWeightLoss}.\footnote{Robust regression is less sensitive 
 	to outliers in the dataset.}
Using the $l_p$ norm as a constraint, the optimization problem here is stated as follows:
\begin{equation} \label{eqn:robust}
\begin{aligned}{}
& \underset{w \in \mathbb{R}^{d}, b \in \mathbb{R}}{\min}  \text{ }  \sum_i \phi\left(w^Tx_i + b, y_i\right) \\
& \text{ s.t. } \| w \|_p \leq r.
\end{aligned}
\end{equation}
 
\end{itemize}

\ph{Datasets} We ran our experiments using several datasets of different sizes and dimensionalities:
\begin{itemize} 
    \item For regression tasks, we used the YearPredictionMSD dataset \cite{UCIMLRepo},
    comprised of 515,345 observations and 90 real-valued features (428 MB in size).
    The regression goal is to predict the year a song was released based on its audio features.

\ignore{
\item \delete{For robust regression tasks, we used the NOxEmissions dataset \cite{NeedCitationForNOxEmissions},
    which contains  8,088 observations and 4  real-valued features. 
    The goal here is to predict an hourly measurement of
    $NOx$ pollution level in the ambient air
(robust regression is often applied to small datasets, as such datasets are more susceptible to outliers~\cite{CiteRobustBaseGuide}). }
\jarrid{We didn't actually use NOxEmissions, remove.}
}

\item   For classification tasks, we used the well-known Adult dataset \cite{UCIMLRepo}, 
	which offers various demographics (14 features) about 48,842 individuals across the United States. 
	The goal is to predict whether an individual earns more than \$50K per year, given 
	 	his/her demographics. 

    \item   For matrix completion tasks, we used two versions of the MovieLens dataset~\cite{movielensDataset}, 
       one with  100K observations of movie ratings from 943 users on 1682 movies, 
       and the other with 1M observations of movie ratings from 6,040 users on 3,900 movies.
\end{itemize}

\ph{Compared Methods} We compared different variants of both gradient descent as well as FW optimization: 
 
\begin{itemize}

\item \textbf{Standard Gradient Descent (GD).} In the $k^{th}$ iteration, GD moves the opposite direction of the gradient:
    \begin{equation}
    	w^{(k+1)} = w^{(k)} - \eta \sum_{x_i} \nabla f(x_i)
    \end{equation}
    where $f(x_i)$ is the loss on data point $x_i$.

\item \textbf{Stochastic Gradient Descent (SGD).} Unlike GD, SGD uses only one data point in each iteration:
    \begin{equation}
    	w^{(k+1)} = w^{(k)} - \eta \nabla f(x_{k+1})
    \end{equation}

\item \textbf{Standard Frank-Wolfe with Predefined Learning Rate (FWPLR).} 
    This variant of Frank-Wolfe corresponds to Algorithm~\ref{alg:fw} with option (A).
    
\item \textbf{Standard Frank-Wolfe with Line Search (FWLS).}  
    This  corresponds to Algorithm~\ref{alg:fw} with option (B). 

\ignore{
\item \textbf{Stochastic Frank-Wolfe (SFW).}  
     This variant of the standard FW
	 uses a random sample of the dataset at each iteration 
     to compute the gradient, which is then passed to a linear oracle to compute
     $max_{v \in \Omega} \langle v, d \rangle$, given some input direction vector $d$~\cite{reddi2016stochastic}. 
      \jarrid{We also still don't have experiments for this (ask Jun Kun?)}
      \barzan{if we dont have any results for this, then remove it. but PA better be faster than both GD and SGD}
}

\item \textbf{Primal Averaging (PA).} Primal averaging is the variant of FW algorithm, which we advocate in this paper.
    This algorithm was presented in Algorithm~\ref{alg:Nest2}.

\item \textbf{Stochastic Primal Averaging (SPA).} This is the stochastic version of PA, as described in Section~\ref{sec:spa}.
 
\end{itemize}

\subsection{Projection Overhead in Gradient Descent}
\label{app:sec:overhead}

\begin{figure}[t]
\begin{center}
\includegraphics[width=0.5\textwidth]{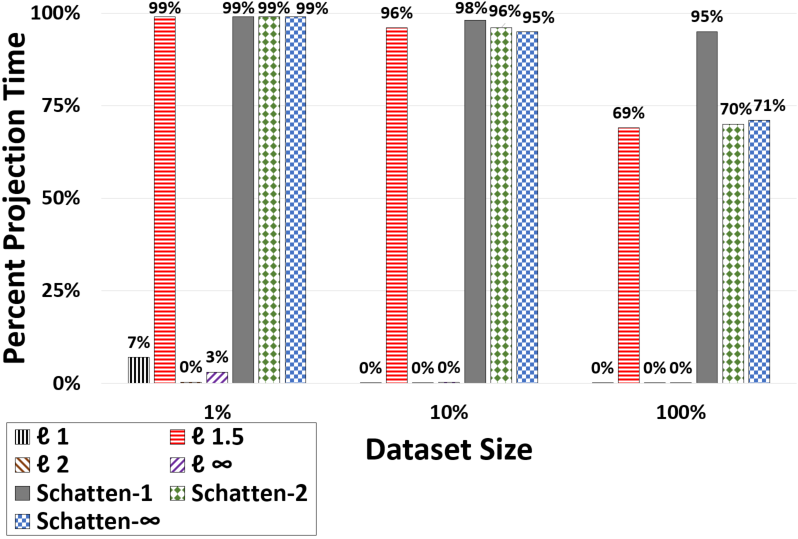}
\caption{The amount of overall time spent on projection in gradient descent.}
\label{fig:percGDOptTimeInProjStep}
\end{center}
\end{figure}

\begin{figure}[t]
\begin{center}
\includegraphics[width=0.5\textwidth]{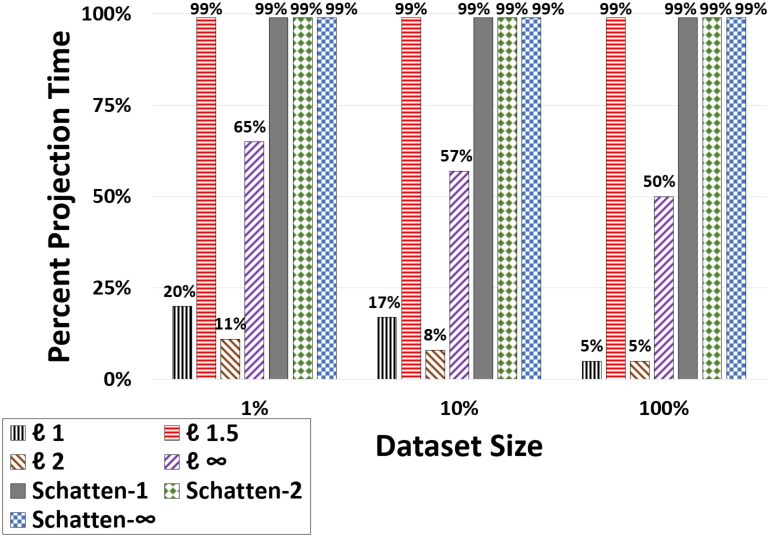}
\caption{The amount of overall time spent on projection in stochastic gradient descent.}
\label{fig:percSGDOptTimeInProjStep}
\end{center}
\end{figure}

    To understand the overhead of the projection step in gradient descent algorithm, we experimented with 
	various machine learning tasks and constraint sets.
    Specifically, we studied $l_1$, $l_{1.5}$, $l_2$, $l_\infty$, Schatten-$1$, Schatten-$2$, and Schatten-$\infty$ norms as 
 	our constraint sets. 
    We used the $l_p$ norm balls in a logistic loss classifier (Adult dataset) and used the Schatten-$p$ norms 
 	in a matrix completion task (MovieLens dataset) with observed quadratic loss (see Section~\ref{app:sec:expr:setup} and Table~\ref{app:tbl:obj_func}). 
To study the effect of data size, we also ran each experiment using different portions of its respective dataset: 1\%, 10\%, and 100\%.

These constraint sets can be divided into three categories~\cite{H14}:
(i) projection onto the $l_1$, $l_2$, and $l_\infty$ balls have a closed-form and thus can be computed efficiently,
(ii) projection onto Schatten-$1$ (a.k.a. nuclear or trace norm), Schatten-$2$, and Schatten-$\infty$ norms has a closed-form 
     but the closed-form requires the SVD of the model matrix,
            and is thus costly,
	and (iii) projection onto $l_{1.5}$ balls does not have any closed-form and requires 
		solving another optimization problem.

Figure~\ref{fig:percGDOptTimeInProjStep} shows the average portion of the total time spent in each iteration of 
	the gradient descent in performing the projection step. 
As expected, the projection step did not account for much of the overall runtime when there was an efficient closed-form, 
i.e., less than $7\%$, $0.03\%$, and $3\%$ for the $l_1$, $l_2$, and $l_\infty$ norms, respectively.
In contrast, 
	projections that involved a costly closed-form or required solving a separate optimization problem
		introduced a significant overhead. 
Specifically, the projection time was responsible for 69--99\% of the overall runtime
	for $l_{1.5}$, 
	 $95$--$99\%$ for Schatten-$1$, $70$--$99\%$ for Schatten-$2$, and 
      	$71$--$99\%$ for Schatten-$\infty$.

Another observation is that this overhead decreased with the data size. 
	This is expected, as the cost of computing the gradient itself in GD grows with the 
		data size and becomes the dominating factor, hence reducing the relative ratio of the projection time
			to the overall runtime of each iteration. 
This is why, for massive datasets, stochastic gradient descent (SGD) is much more popular
	than standard GD~\cite{hogwild,reAsyncSGD,schattenUsedOft3}. 
Therefore, we measured the projection overhead for SGD as well.
However, since SGD's runtime does not 
	depend on the overall data size,
		we did not vary the dataset size.
		
The results for SGD are shown in Figure~\ref{fig:percSGDOptTimeInProjStep}.
The trend here is similar to GD, but the overheads are more pronounced. 
Constraint sets without a computationally efficient projection caused a significant overhead in SGD.  
However, for SGD, even the projections with
efficient closed-form solutions introduced a noticeable overhead: $5$--$20\%$ for $l_1$, $5$--$11\%$ for $l_2$, and $50$--$65\%$ for $l_\infty$.
While SGD takes significantly less time than GD to compute its descent direction for large datasets, the time to compute the projection
remains constant.  Hence, the fraction of the overall computation time spent on projection 
    is larger in SGD than in GD.
    
The reason for the particularly higher overhead in case of $l_\infty$ is that
projecting onto an $l_\infty$ ball cannot be vectorized. In other words, projection onto $l_1$ and $l_2$ balls 
	can 	 take better advantage of the underlying hardware than projection onto an $l_\infty$ ball, causing the observed disparity in runtimes.
In summary, the projection overhead is a major concern for both GD and SGD, whenever there is no efficient closed-form.
	Furthermore, this problem is still important for SGD, even when there is an efficient procedure for projection.

\subsection{Primal Averaging's Convergence Rate}
\label{app:sec:PAConv}

\begin{figure}[!h]
\centering     
\begin{subfigure}[t]{0.45\linewidth}
\centering
    \includegraphics[width=\linewidth, height=50mm, keepaspectratio]{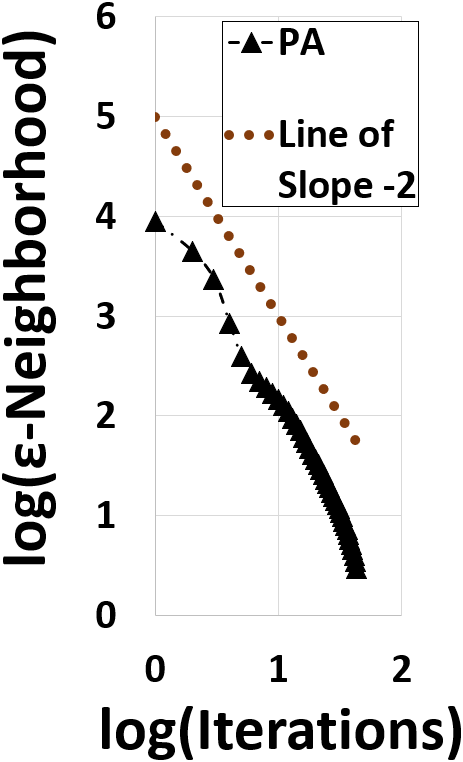}
    \caption{Classification with $l_2$ norm.}
    \label{fig:L2PAConvRate:app}
\end{subfigure}
\hspace{0.5cm}
\begin{subfigure}[t]{0.45\linewidth}
\centering
    \includegraphics[width=\linewidth, height=50mm, keepaspectratio]{plots/PrimalAveragingMtxCompLogConvRate.png}
    \centering
    \caption{Matrix completion with Schatten-$2$ norm.}
    \label{fig:MtxCmpPAConvRate:app}
\end{subfigure}
\caption{Convergence rate of PA on classification and matrix completion tasks.
} 
\label{fig:PAConvRate} 
\vskip -0.1in
\end{figure}


In this section, we report  experiments on various   
machine learning tasks and datasets to compare PA's performance against 
other variants of FW when solving (smooth)
convex functions with strongly convex  constraint sets. 
 In particular,
	we studied the performance of PA, FWLS, and FWPLR for both $l_2$ and Schatten-$2$ balls as our strongly convex constraint sets (see Section~\ref{sec:prelim}
for a discussion of why these constraints are strongly convex).  
We used the $l_2$ norm ball
for a logistic classifier on the Adult dataset, as well as a linear regression task on the YearPredictionMSD dataset. 
We used the Schatten-$2$ norm ball for a matrix completion task on the MovieLens dataset.

 First, we measured the $\epsilon$-neighborhood of the global
	minimum reached by PA at each iteration.
Our theoretical results (Theorem~\ref{thm:convex})
	predict a convergence rate of 
	$O( \frac{1}{t^2} )$ in this case.
To confirm this empirically, we plotted the logarithm of the  $\epsilon$-neighborhood
	against the logarithm of the iteration number.
	If the convergence rate of  $O( \frac{1}{t^2} )$ were to hold,
		we would expect a straight line with a slope of $-2$ after taking the logarithms.
	 
The plots are shown in
	Figures~\ref{fig:L2PAConvRate:app} and~\ref{fig:MtxCmpPAConvRate:app}	
	 for the classification and matrix completion
tasks, respectively.  The results confirm our theoretical results, as the plots 
exhibit a slope of -2.34 and -2.41, respectively. 
	Note that a slightly steeper slope is expected in practice, since 
    our theoretical results only provide a worst-case upper bound on the convergence rate.
\ignore{When a power function
was fit to the data in Figure~\ref{fig:PAConvRate} we found equations of
$O( \frac{1}{t^{2.344}} )$ and $O( \frac{1}{t^{2.41}} )$ for the classification
and matrix completion task, respectively.}

\subsection{Primal Averaging's Performance versus Other FW Variants}
\label{app:sec:PAPerf}

\begin{figure}[t]
\centering     
\begin{subfigure}[t]{0.75\linewidth}
\centering
    \includegraphics[width=1\linewidth, keepaspectratio]{plots/RegressionLossFunction.PNG}
    \caption{Regression with $l_2$ norm.}
    \label{fig:L2Time:app}
\end{subfigure}
\begin{subfigure}[t]{0.77\linewidth}
\centering
    \includegraphics[width=1\linewidth, keepaspectratio]{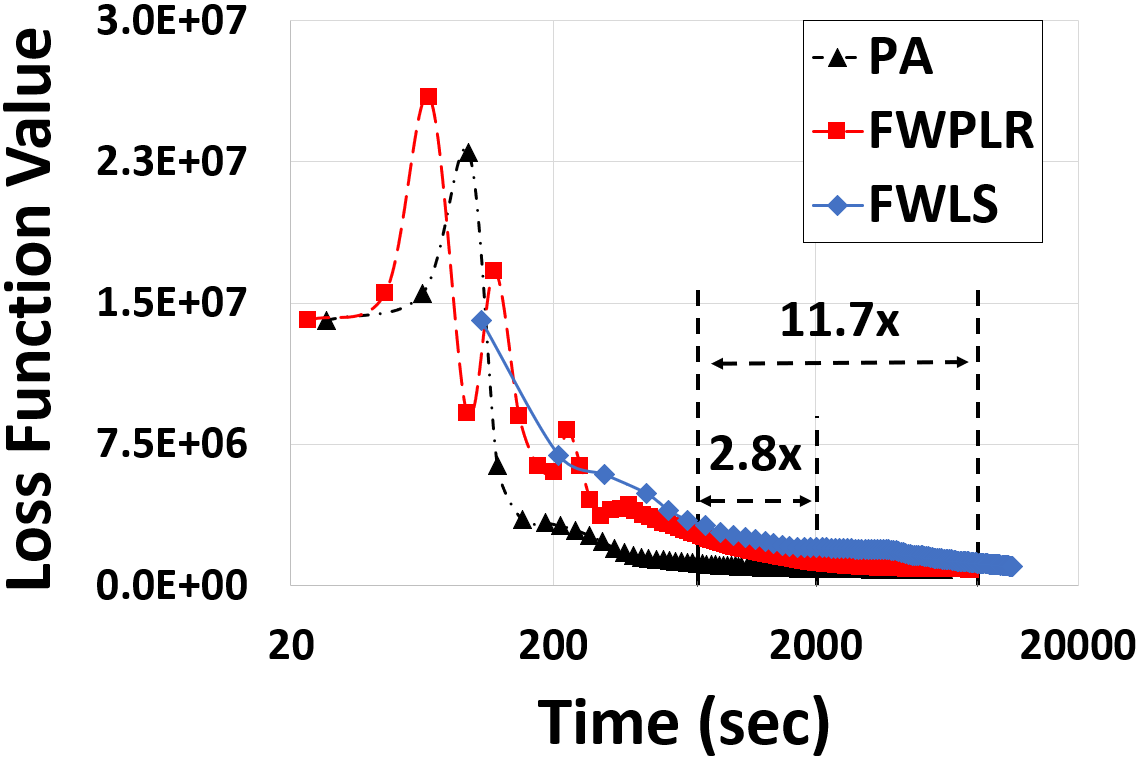}
    \centering
    \caption{Matrix completion with Schatten-$2$ norm.}
    \label{fig:MtxCmpTime}
\end{subfigure}
\caption{Performance of different Frank-Wolfe variants on regression and matrix completion tasks.}
\label{fig:FWTimes}
\vskip -0.1in
\end{figure}

\begin{figure}[t]
\begin{center}
\includegraphics[width=0.3\textwidth]{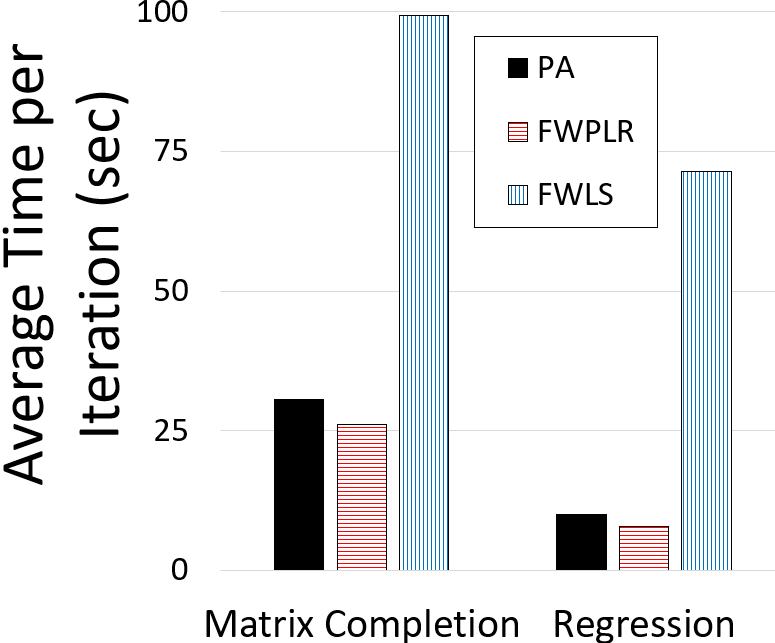}
\caption{Average iteration time of PA, FWPLR, and FWLS on regression and matrix completion tasks.}
\label{fig:AvgIterTimeFWVars}
\end{center}
\end{figure}

\begin{figure}[t]
\centering
\begin{subfigure}[t]{0.6\linewidth}
    \centering
    \includegraphics[width=1\linewidth, keepaspectratio]{plots/PAVsGDLP20.PNG}
    \caption{PA versus GD.}
    \label{fig:DetSPAVsSGD:app}
\end{subfigure}

\begin{subfigure}[t]{0.6\linewidth}
    \centering
    \includegraphics[width=1\linewidth, keepaspectratio]{plots/SGDBetterThanSPA.PNG}
    \caption{SPA versus SGD with $l_2$ norm.}
    \label{fig:SGDGtSPA:app}
\end{subfigure}

\begin{subfigure}[t]{0.6\linewidth}
    \centering
    \includegraphics[width=1\linewidth, keepaspectratio]{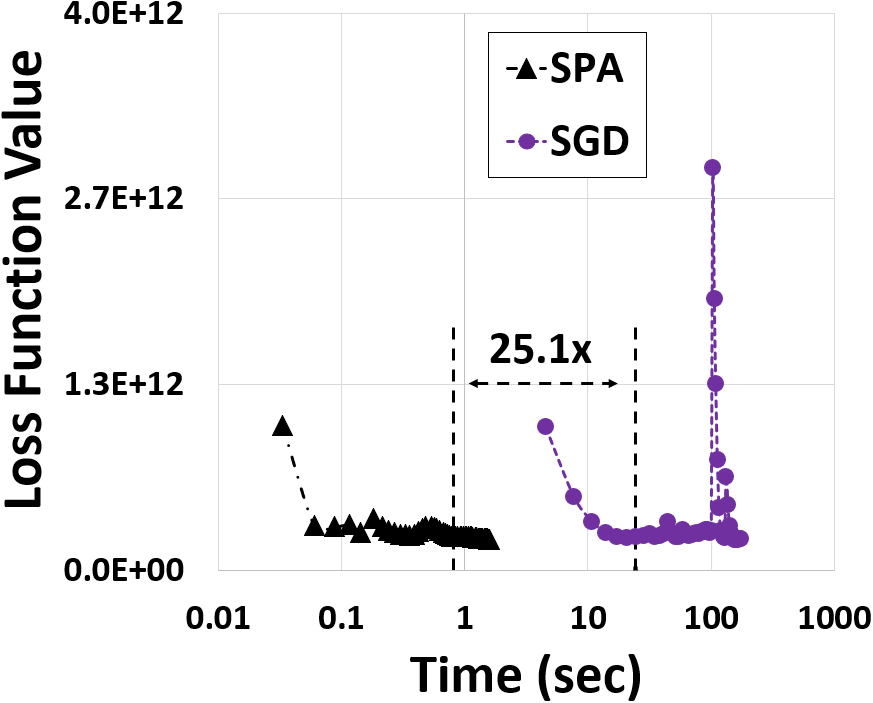}
    \caption{SPA vs SGD with $l_{1.1}$ norm.}
    \label{fig:SPAGtSGD}
\end{subfigure}
\caption{Performance of PA  versus GD (and their stochastic versions, SPA and SGD) on the regression task.}
\label{fig:PAVsGD}
\end{figure}

To compare the actual performance of various FW variants on convex functions,
	we used the same settings as Section~\ref{app:sec:PAConv}.
However, instead of the number of iterations to convergence, 
	this time we measured the actual runtimes. 

Here, we compared all three variants: PA, FWLS, and FWPLR.
Figures~\ref{fig:L2Time:app} and~\ref{fig:MtxCmpTime} report the time taken to achieve each value of the loss
function for the regression and matrix completion tasks, respectively.
    To compare the performance of these algorithms, we measured the difference between
    the time it took for each of them to converge.  
    To determine convergence, here picked the first iteration at which 
    		the loss value was within $\pm 2\%$ of the previous loss value (practical convergence),
			and was also within $\pm 2\%$ of the global minimum (actual convergence). 
	The first time at which these iterations were reached for each algorithm are marked by  vertical, striped lines in Figures~\ref{fig:L2Time:app} 	
    and~\ref{fig:MtxCmpTime}.
 
 
For the regression task, PA converged 
 $3.7\times$ and $15.6\times$  faster than FWPLR and FWLS, respectively.  
For the matrix completion task, PA converged $2.8\times$  and $11.7\times$ faster 
than FWPLR and FWLS, respectively.
These considerable speedups have significant ramifications in practice.
Traditionally, PA has been shied  away from, simply because it is slower in each iteration 
(due to PA's use of auxiliary sequences and its extra summation step),
while its convergence rate was believed to be the same as the more efficient variants~\cite{L13}.
However, as we formally proved in Section~\ref{sec:convex}, 
	PA does indeed converge within much fewer iterations.
Thus, the results shown in 
	Figures~\ref{fig:L2Time:app} and~\ref{fig:MtxCmpTime}
		validate our hypothesis that PA's faster convergence rate more than compensates for its additional computation at each iteration. 
Figure~\ref{fig:AvgIterTimeFWVars} 
 reports the  per-iteration cost of these FW variants on average, 
 showing that PA is only 1.2--1.3x slower than than FWPLR in each iteration. 
This is why PA's much faster convergence rate leads to much better performance in practice, compared to FWPLR.
On the other hand, 
	although FWLS offers the same convergence rate as PA, 
    PA's cost per iteration is $3.2$--$7.1\times$ faster than FWLS,
			which also explains PA's superior performance over FWLS.

Finally, we note that PA's improvements were much more drastic for the regression
task than the matrix completion task ($3.7$--$15.6\times$ versus $2.8$--$11.7\times$).
 This is due to of the following reason.
     We recall that the Schatten-$p$ norm ball with radius $r$ is $\alpha$-strongly convex for 
    $p \in (1, 2]$ and with $\alpha = \frac{p - 1}{r}$.
  The matrix completion task on the MovieLens dataset requires us to predict the
values of a $6,040 \times 3,900$ matrix ($6,040$ users and $3,900$ movies).  Thus, to be able
to maintain a reasonable number of potential matrices within our constraint set, we had to set 
$r = 12000$, namely $\alpha=\frac{1}{12000}$.  
According to Theorem~\ref{thm:convex}, the convergence rate is $O( \frac{L}{\alpha^2 g^2 t^2} )$,
which is why a small value of $\alpha$ slows down PA's convergence.

\subsection{Primal Averaging's Performance versus Projected Gradient Descent}
\label{app:sec:PAVsGD}

    In this section, we compare the performance of PA and projected gradient descent.  
    We evaluated deterministic and stochastic versions of both algorithms on the regression task 
    with the same settings as in Section~\ref{app:sec:PAPerf}.  
    	  We used the same methodology to determine
    convergence as in Section~\ref{app:sec:PAPerf}.

	The results are shown in     Figure~\ref{fig:PAVsGD}.
    As expected, PA significantly outperformed projected GD, converging
    $7.7\times$ faster (Figure~\ref{fig:DetSPAVsSGD:app}).  
    To better compare their stochastic versions (SPA and SGD), however, 
  	 we used two  different settings.  The first used
    the $l_2$ ball as the constraint set, as an example of a case with an efficient projection, 
    and the second used the $l_{1.1}$ ball as an example of a case with a costly projection.
     
    The results   conformed with our expectation again.
    When the projection onto the constraint set was efficient,
    SGD converged $4.6\times$  faster than SPA 
    (Figure~\ref{fig:SGDGtSPA}).  On the other hand, when the projection   was costly,
    SPA far outperformed SGD, converging $25.1\times$ faster (Figure~\ref{fig:SPAGtSGD}).

\subsection{Frank-Wolfe for (Smooth) Strictly-Locally-Quasi-Convex Functions}
\label{sec:FWSLQCExp}

\begin{figure}[!htbp]
\centering     
\begin{subfigure}[t]{0.4\linewidth}
\centering
    \includegraphics[height=50mm, keepaspectratio]{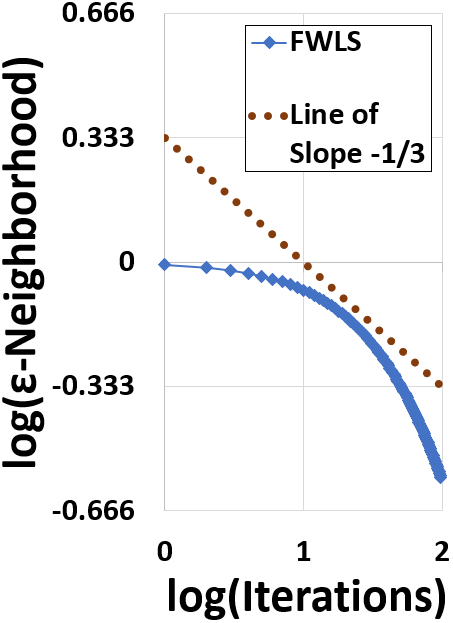}
    \centering
    \caption{Convergence rate when \newline $\epsilon < 1$.}
    \label{fig:SLQCEpsilonRateEpsLeq1}
\end{subfigure}
\hspace*{1.5cm}
\begin{subfigure}[t]{0.4\linewidth}
\centering
    \includegraphics[height=50mm, keepaspectratio]{plots/SLQCEpsGreater1VsIter.PNG}
    \caption{Convergence rate when \newline $\epsilon > 1$.} 
    \label{fig:SLQCEpsilonRateEpsGeq1}
\end{subfigure}
\caption{Iterations for FWLS to converge to an $\epsilon$-neighborhood when optimizing a strictly-locally-quasi-convex function.}
\label{fig:SLQCEpsilonRate}
\end{figure}

According to Theorem~\ref{thm:quasi},  even when the loss function is not convex, 
	FWLS  still converges (to an $\epsilon$-neighborhood
    of the global minimum) within $O(max( \frac{1}{\epsilon^2}, \frac{1}{\epsilon^3}))$ iterations,
as long as the loss function is strictly-locally-quasi-convex. 
   To verify this empiricially, we used the squared Sigmoid loss
   	function\footnote{See Section~\ref{app:sec:expr:setup} for a discussion of the strictly-locally-quasi-convexity of squared Sigmoid loss.}
    	for a classification
    task (Adult dataset) with an $l_2$ ball as our constraint set.

Note  that, to conform with our theoretical result, 
    FWLS must exhibit an $O\left( \frac{1}{t^{1 / 2}} \right)$ convergence rate when $\epsilon > 1$ and an
    $O\left( \frac{1}{t^{1 / 3}} \right)$ convergence rate when $\epsilon < 1$.
       To better illustrate this difference, we examine two plots:
      Figure~\ref{fig:SLQCEpsilonRateEpsGeq1} displays the iterations where $\epsilon > 1$,
      while Figure~\ref{fig:SLQCEpsilonRateEpsLeq1} displays the iteration where $\epsilon < 1$.  
      Both plots show the logarithm of the $\epsilon$-neighborhood against the logarithm of the
      iteration number.  This means we should expect to see the loss values decreasing at a slope steeper than or equal to $-\frac{1}{2}$ and
      $-\frac{1}{3}$ in Figure~\ref{fig:SLQCEpsilonRateEpsGeq1} and Figure~\ref{fig:SLQCEpsilonRateEpsLeq1}, respectively.
 
       The plots confirm our theoretical results, exhibiting a slope of $-2.12$ when $\epsilon > 1$ and 
      $-0.377$ when $\epsilon < 1$.  Note that the steeper slopes here are expected,  as 
      our theoretical results only provide a worst-case upper bound on the convergence rate.
      Notably, FWLS showed a significantly steeper slope when $\epsilon > 1$.
      We observe that the convergence rate bound for $\epsilon > 1$
      is missing the smoothness parameter $L$ of the $\epsilon < 1$ bound.  It is noted in~\cite{HLS15}
      that using the squared Sigmoid loss is equivalent to the perceptron problem with a $\gamma$-margin, and
      Kalai and  Sastry~\cite{SLQCExpLossIsPerceptron} show that the smoothness parameter $L$ of  the latter
       is $\frac{1}{\gamma}$.
       \footnote{The     $\gamma$-margin can intuitively be thought of as the difficulty of the classification task, 
        	with a larger margin indicating an easier classification.}
      Thus, when the margin is large, the $\epsilon < 1$ case is able to converge at a rate closer
      to $O\left( \frac{1}{t^{1 / 3}} \right)$ than the $\epsilon > 1$ case to its rate of
      $O\left( \frac{1}{t^{1 / 2}} \right)$. Thus, we hypothesize that it is the absence of 
      the $L$ factor in the upper bound for $\epsilon > 1$, which
      primarily contributes to its convergence rate being faster than $- \frac{1}{2}$   when $\epsilon > 1$.
  
\begin{figure}
\inv
\centering
\includegraphics[height=50mm, keepaspectratio]{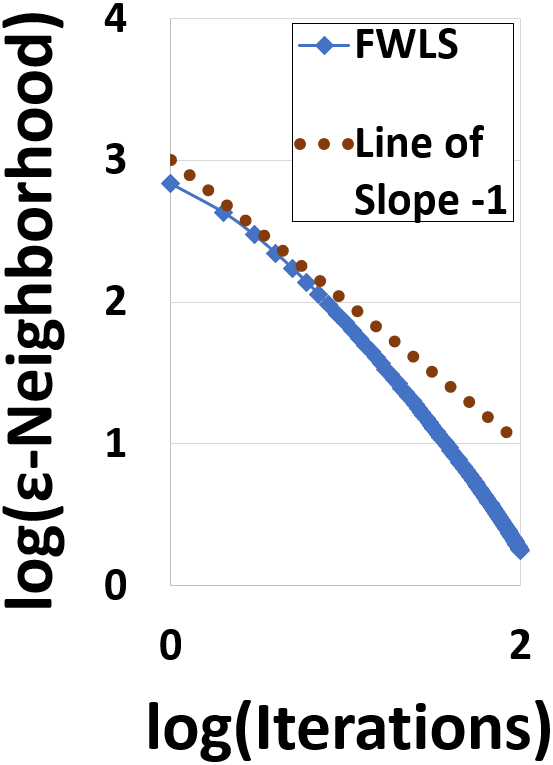}
\centering
\caption{Performance of FWLS for regression with a non-convex loss.}
\label{fig:FWLSNonConvex:app}
\end{figure}

\subsection{Frank-Wolfe for (Smooth) Non-Convex Functions}
\label{sec:NonCVExp}

In Theorem~\ref{thm:nonconvex}, we proved that FWLS  converges  to a stationary
    point of (smooth) non-convex functions at a rate of $O( \frac{1}{t} )$,
  as long as it is constrained to a strongly convex set.  
To empirically verify whether this upper bound is tight, 
	we use the bi-weight loss (see Section~\ref{app:sec:expr:setup}
    and Table~\ref{app:tbl:obj_func}) in a classification task (Adult dataset) with an $l_2$ ball constraint. 

    In Figure~\ref{fig:FWLSNonConvex:app}, we measured the $\epsilon$-neighborhood reached by FWLS at each iteration, plotting 
    the logarithm of the $\epsilon$-neighborhood against the logarithm of the iteration number.  To confirm the 
    $O( \frac{1}{t} )$ convergence rate found in Theorem~\ref{thm:nonconvex}, we expect to see a straight line of
    slope $-1$ in Figure~\ref{fig:FWLSNonConvex:app}.

    The empirical results confirm our theoretical results, showing a slope of $-1.46$.  Again, we note that
    a steeper slope is expected in practice as Theorem~\ref{thm:nonconvex} only provides a worst-case upper bound
    on the convergence rate.

%

\end{document}